\documentclass{article}

\usepackage{microtype}
\usepackage{graphicx}
\usepackage{subfigure}
\usepackage{booktabs} %

\usepackage{hyperref}

\usepackage[accepted]{icml2023}

\usepackage{amsmath}
\usepackage{amssymb}
\usepackage{mathtools}
\usepackage{amsthm}

\usepackage{url}
\usepackage{microtype}
\usepackage{graphicx}
\usepackage{float}
\usepackage{color}
\usepackage{url}
\usepackage{rotating}
\usepackage{adjustbox}
\usepackage{layouts}
\usepackage{wrapfig}
\usepackage{tablefootnote}
\usepackage{etoolbox}
\usepackage{booktabs}
\usepackage{multirow}
\usepackage{multicol}
\usepackage{array}
\usepackage{enumitem}
\newcolumntype{C}{>{$}c<{$}}
\newcolumntype{L}{>{$}l<{$}}
\newcolumntype{R}{>{$}r<{$}}

\usepackage{amsmath}
\usepackage{amsfonts}
\usepackage{bm}
\usepackage{bbm}
\usepackage{mathtools}
\usepackage{thmtools}
\usepackage{xspace}
\usepackage{nicefrac}
\usepackage{makecell}

\usepackage[capitalise,nameinlink,noabbrev]{cleveref}
\crefname{section}{Sec.}{Secs.}
\crefname{appendix}{App.}{Apps.}
\crefname{algorithm}{Alg.}{Algs.}
\creflabelformat{equation}{#2\textup{#1}#3}  %

\usepackage{tikz}
\usetikzlibrary{positioning}
\usetikzlibrary{calc}
\usetikzlibrary{arrows.meta}
\usetikzlibrary{backgrounds}
\usetikzlibrary{decorations.pathmorphing, patterns}
\usetikzlibrary{decorations.pathreplacing}
\usetikzlibrary{plotmarks}

\usepackage{contour}

\definecolor{lightgray}{gray}{0.85}
\definecolor{lightlightgray}{gray}{0.9}
\definecolor{C1}{HTML}{1F77B4}
\definecolor{C2}{HTML}{FF7F0E}
\definecolor{C3}{HTML}{2CA02C}
\definecolor{C4}{HTML}{D62728}
\definecolor{C5}{HTML}{9467BD}
\colorlet{C1light}{C1!70!white}
\colorlet{C2light}{C2!70!white}
\colorlet{C3light}{C3!70!white}
\colorlet{C4light}{C4!70!white}
\colorlet{C5light}{C5!70!white}
\colorlet{C1lighter}{C1!40!white}
\colorlet{C2lighter}{C2!40!white}
\colorlet{C3lighter}{C3!40!white}
\colorlet{C4lighter}{C4!40!white}
\colorlet{C5lighter}{C5!40!white}
\colorlet{C1vlight}{C1!20!white}
\colorlet{C2vlight}{C2!20!white}
\colorlet{C3vlight}{C3!20!white}
\colorlet{C4vlight}{C4!20!white}
\colorlet{C5vlight}{C5!20!white}

\usepackage[normalem]{ulem}

\usepackage{pifont}

\newcommand{\param}{\ensuremath{\vw}\xspace}
\newcommand{\eparam}{\ensuremath{w}\xspace}

\newcommand{\paramstar}{{\tilde{\param}}}

\newcommand{\augmentation}{\ensuremath{\boldsymbol{\eta}}\xspace}
\newcommand{\hyperparam}{\ensuremath{\vh}\xspace}

\newcommand{\vfaug}{\ensuremath{\hat{\vf}\xspace}}

\newcommand{\ggn}{\textsc{ggn}\xspace}
\newcommand{\ntk}{\textsc{ntk}\xspace}
\newcommand{\lap}{\textsc{la}\xspace}

\newcommand{\kfac}{\textsc{kfac}\xspace}

\newcommand{\map}{\textsc{MAP}\xspace}

\newcommand{\sminus}{\text{-}\xspace}

\newcommand{\vflin}{\ensuremath{\vf^\text{lin}_{\paramstar}}}

\newcommand{\iid}{\emph{i.i.d.}}

\newcommand{\wrt}{w.r.t.\xspace}

\newtheorem{theorem}{Theorem}
\newtheorem{corollary}{Corollary}[theorem]
\newtheorem{lemma}{Lemma}

\newtheorem{example}{Example}[theorem]

\def\1{\bm{1}}

\newcommand{\transpose}{^\mathrm{\textsf{\tiny T}}}

\newcommand{\inv}{^{-1}}
\newcommand{\defeq}[0]{\mathrel{\stackrel{\textnormal{\tiny def}}{=}}}
\newcommand{\E}{\mathbb{E}}

\newcommand{\R}{\mathbb{R}}

\def\sA{{\mathcal{A}}}
\def\sB{{\mathcal{B}}}

\def\sD{{\mathcal{D}}}

\def\sI{{\mathcal{I}}}

\def\sL{{\mathcal{L}}}

\def\sP{{\mathcal{P}}}

\def\sR{{\mathcal{R}}}
\def\sS{{\mathcal{S}}}

\def\sU{{\mathcal{U}}}

\def\D{{\sD}}
\def\data{{\sD}}

\def\vzero{{\mathbf{0}}}

\def\vf{{\mathbf{f}}}
\def\vg{{\mathbf{g}}}
\def\vh{{\mathbf{h}}}

\def\vw{{\mathbf{w}}}
\def\vx{{\mathbf{x}}}
\def\vy{{\mathbf{y}}}

\def\eps{{\epsilon}}
\def\vepsilon{{\boldsymbol{\eps}\xspace}}

\def\mA{{\mathbf{A}}}
\def\mB{{\mathbf{B}}}
\def\mC{{\mathbf{C}}}

\def\mH{{\mathbf{H}}}
\def\mI{{\mathbf{I}}}
\def\mJ{{\mathbf{J}}}
\def\mK{{\mathbf{K}}}

\def\mM{{\mathbf{M}}}

\def\mP{{\mathbf{P}}}

\def\mLambda{{\boldsymbol{\Lambda}\xspace}}

\def\evf{{\textnormal{f}}}

\def\emM{{\textnormal{M}}}

\DeclareMathAlphabet{\mathsfit}{\encodingdefault}{\sfdefault}{m}{sl}
\SetMathAlphabet{\mathsfit}{bold}{\encodingdefault}{\sfdefault}{bx}{n}

\newcommand{\half}{\mbox{$\frac{1}{2}$}}

\newcommand{\given}{\mbox{$|$}}

\newcommand{\la}{\mbox{$\leftarrow$}}

\newcommand{\order}[1]{\ensuremath{\mathcal{O}(#1)}}

\newcommand{\simiid}{\overset{{\scriptscriptstyle\mathrm{iid}}}{\sim}}

\icmltitlerunning{Stochastic Marginal Likelihood Gradients using Neural Tangent Kernels}

\begin{document}

\twocolumn[
\icmltitle{Stochastic Marginal Likelihood Gradients using Neural Tangent Kernels}

\begin{icmlauthorlist}
\icmlauthor{Alexander Immer}{eth,mpi}
\icmlauthor{Tycho F.A. van der Ouderaa}{imperial}
\icmlauthor{Mark van der Wilk}{imperial}\\
\icmlauthor{Gunnar R\"atsch}{eth}
\icmlauthor{Bernhard Sch\"olkopf}{mpi}
\end{icmlauthorlist}

\icmlaffiliation{eth}{Department of Computer Science, ETH Zurich, Switzerland}
\icmlaffiliation{mpi}{Max Planck Institute for Intelligent Systems, T\"ubingen, Germany}
\icmlaffiliation{imperial}{Imperial College London, UK}

\icmlcorrespondingauthor{Alexander Immer}{alexander.immer@inf.ethz.ch}

\icmlkeywords{Machine Learning, ICML, marginal likelihood, Bayesian model selection, deep learning, approximate inference, neural tangent kernel, evidence lower bound, Kronecker-factored approximate curvature}

\vskip 0.3in
]

\printAffiliationsAndNotice{}  %

\begin{abstract}
Selecting hyperparameters in deep learning greatly impacts its effectiveness but requires manual effort and expertise. Recent works show that Bayesian model selection with Laplace approximations can allow to optimize such hyperparameters just like standard neural network parameters using gradients and on the training data. However, estimating a single hyperparameter gradient requires a pass through the entire dataset, limiting the scalability of such algorithms. In this work, we overcome this issue by introducing lower bounds to the linearized Laplace approximation of the marginal likelihood. In contrast to previous estimators, these bounds are amenable to stochastic-gradient-based optimization and allow to trade off estimation accuracy against computational complexity. We derive them using the function-space form of the linearized Laplace, which can be estimated using the neural tangent kernel. Experimentally, we show that the estimators can significantly accelerate gradient-based hyperparameter optimization.

\end{abstract}

\section{Introduction}
\label{sec:introduction}
Deep learning models have many hyperparameters that need to be chosen properly to achieve good performance.
For example, making the right architecture, regularization, or data augmentation choices remains a problem that requires significant human effort and computation.
This is even apparent within the field of deep learning itself, where the state-of-the-art settings change frequently. %
Such manual trial-and-error procedure is inefficient and can hinder the application of deep learning to new problems.
A more automatic procedure could be much more sustainable.

\begin{figure}[t]
    \centering
    \includegraphics{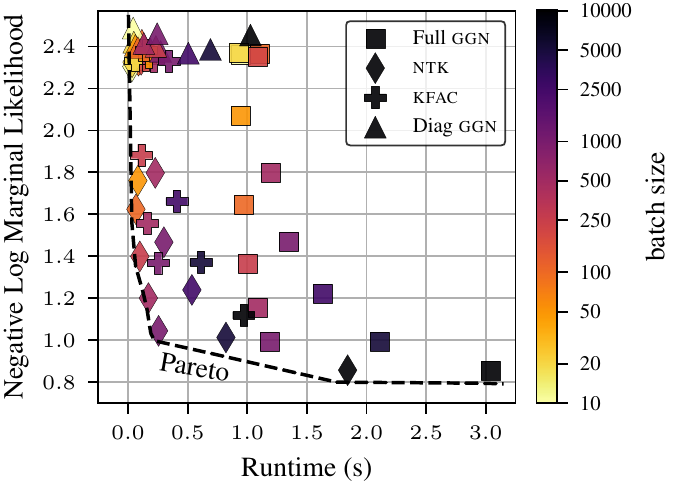}
    \vspace{-1em}
    \caption{Pareto frontier between marginal likelihood estimator tightness and runtime for invariance learning on rotated MNIST.
    Our proposed estimators enable stochastic approximation with varying batch size in contrast to the four previously used full-batch estimators (batch size of $10\,000$ in black).
    \ntk-based estimators on subsets of data perform well at more than $10\times$ speed-up.} %
    \label{fig:pareto}
    \vspace{-1em}
\end{figure}

Automated machine learing (AutoML) methods have the potential to greatly reduce the cost of deploying deep learning for new problems~\citep{hutter2019automated}.
Commonly, black-box algorithms are applied in this setting to optimize hyperparameters towards the best possible validation performance.
For example, neural architectures can be optimized by training each option to convergence and assessing the validation performance repeatedly.
However, such iterative procedures require training many models to convergence and suffer from high-dimensional hyperparameters due to their black-box nature.
We posit that (stochastic-)gradient-based optimization of hyperparameters would be more desirable.

\begin{figure*}[th!]
    \centering
    \includegraphics[width=\textwidth]{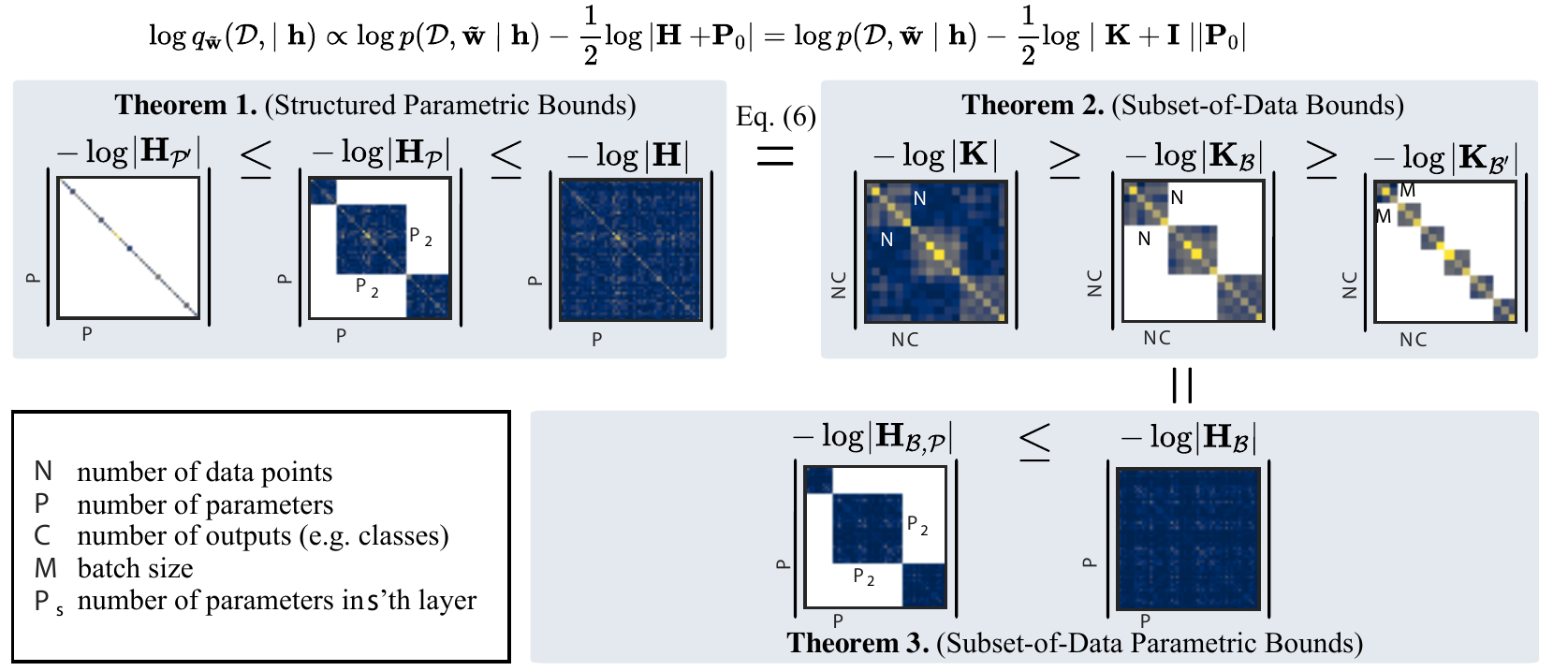}
    \vspace{-2em}
    \caption{Schematic overview of the derived lower bounds to the linearized Laplace marginal likelihood. The parametric bounds on the top left justify commonly used block-diagonal and diagonal Hessian and Gauss-Newton approximations.
    The bounds on the right are derived from the dual \ntk perspective and give a novel family of lower bounds. 
    These lower bounds enable stochastic marginal likelihood estimation and gradients with both \ntk-based estimators and parametric variants, e.g., Gauss-Newton, by translating them back (bottom).}
    \label{fig:overview}
    \vspace{-0.5em}
\end{figure*}

Bayesian model selection, where we consider hyperparameters and neural network weights jointly as part of a probabilistic model, is amenable to gradient-based optimization and also does not require any validation data~\citep{mackay2003information}.  
Gradient-based optimization of hyperparameters tends to suffer less from high dimensionality than iterative black-box methods~\citep{lorraine2020optimizing}.
Further, in such integrated procedure, the hyperparameters can be jointly optimized with the neural network weights~\citep{foresee1997gauss, immer2021scalable}.
While this is theoretically appealing, estimating the hyperparameter gradient is costly as it requires differentiating the \emph{marginal likelihood}, i.e., the normalization constant of the posterior. 

To enable Bayesian model selection of hyperparameters in deep learning with marginal likelihood gradients, \citet{immer2021scalable} recently revisited the use of the Laplace approximation~\citep{mackay1992bayesian} with modern Gauss-Newton approximations.
While these structured linearized Laplace approximations can be used for architecture comparison and gradient-based learning of data augmentations and regularization strength~\citep{immer2021scalable, immer2022probing, immer2022invariance, antoran2021linearised, daxberger2021laplace, hataya2021gradient}, each hyperparameter gradient requires a pass through the entire training dataset. 
In particular, the objective does not allow (unbiased) stochastic gradient estimation.

In the present work, we derive lower bounds to the Laplace approximation of the marginal likelihood that are amenable to stochastic-gradient-based optimization and show that currently used structured approximations are in fact lower bounds~(overview in \cref{fig:overview}).
Instead of a whole pass through the training data, our estimators require only batches of data to estimate hyperparameter gradients, which can greatly increase the scalability of such algorithms.
The size of the batches can further be chosen to achieve a trade-off between estimation quality and runtime complexity~(cf.~\cref{fig:pareto}).
We derive the lower bounds from the dual form of the linearized Laplace~\citep{khan2019approximate, immer2020improving}, which relies on the (empirical) neural tangent kernel~\citep[\ntk,][]{jacot2018neural}.
By partitioning the \ntk using batches, we obtain a lower bound to the linearized Laplace marginal likelihood, which permits unbiased stochastic estimates and gradients. %
Empirically, our estimators enable up to 25-fold acceleration of hyperparameter optimization and application to larger datasets. %

\section{Background}
\label{sec:background}
We consider supervised learning tasks with inputs $\vx_n \in \R^D$ and targets $\vy_n \in \R^C$ giving dataset $\D = \{(\vx_n, \vy_n)\}_{n=1}^N$ with $N$ pairs.
We model the targets given inputs with a neural network.
We parameterize the neural network $\vf$, as usual, with weights $\param \in \R^P$, but additionally make the dependence on eventual hyperparameters $\hyperparam$, such as architecture, explicit, and have $\vf(\vx;\param,\hyperparam) \in \R^C$.
Assuming the data are \iid, we can define a probabilistic model with likelihood $p(\data \given \param, \hyperparam) = \prod_{n=1}^N p(\vy_n \given \vx_n, \param, \hyperparam)$ that depends on the neural network.
Further, we have a prior $p(\param \given \hyperparam)$. 
\subsection{Bayesian Model Selection}
In principle, Bayesian inference makes no distinction between hyperparameters $\hyperparam$ and weights $\param$ and simply infers both jointly.
With an additional prior $p(\hyperparam)$, this means that we obtain the joint posterior $p(\param, \hyperparam \given \data) \propto p(\data \given \param, \hyperparam) p(\param \given \hyperparam) p(\hyperparam)$.
While there exist methods to approximate the posterior $p(\param \given \data, \hyperparam)$ efficiently in deep learning, Bayesian inference of the hyperparameters \hyperparam is complicated due to their heterogeneous and complicated structure, e.g., it could require a distribution over neural architectures.
Instead, it is common to perform \emph{empirical Bayes} to select hyperparameters \hyperparam by optimizing the \emph{marginal likelihood},
\begin{equation}
    \label{eq:marglik}
    p(\data \given \hyperparam) = \int p(\data \given \param, \hyperparam) p(\param \given \hyperparam) \, \mathrm{d} \param \,,
\end{equation}
which implements Occam's razor, performing a trade-off between model complexity and accuracy~\citep{rasmussen2001occam, bishop2006pattern}. 
This point approximation can work well unless the number of hyperparameters is too large, which leads to overfitting~\citep{ober2021promises}.

We can optimize the marginal likelihood by gradient ascent on the differentiable hyperparameters,
\begin{equation}
    \label{eq:marglik-grad}
   \vh_{t+1} \la \vh_{t} + \nabla_\vh \log p (\D \given \vh)\vert_{\vh = \vh_t},
\end{equation}
similarly to how we differentiate the log joint, or regularized log likelihood, to optimize neural network weights $\param$.

\subsection{Linearized Laplace Approximation}
\label{sec:lin_la}
The linearized Laplace (\lap) approximation~\citep{mackay1992bayesian} provides an approximation to the log marginal likelihood and its scalable variants are currently among the few approximations that have proven effective for deep learning hyperparameter optimization.
First, the neural network is linearized at a mode of the posterior $\paramstar$,
\begin{equation}
    \label{eq:linearization}
    \vflin(\vx;\param,\hyperparam) \defeq \vf(\vx;\paramstar,\hyperparam) + \mJ_\paramstar(\vx;\hyperparam) (\param - \paramstar),
\end{equation}
where $\mJ$ denotes the Jacobian of the neural network \wrt weights with entries $[\mJ_\paramstar]_{cp} = \frac{\partial f_c(\vx;\param,\hyperparam)}{\partial \eparam_p}\vert_{\param =\paramstar}$.
Further, we have the Hessian of the negative log likelihood \wrt $\vf$, 
\begin{equation}
    \label{eq:log_lik_hessian}
    \mLambda_\paramstar(\vx; \hyperparam) \defeq -\nabla_{\vf}^2 \log p(\vy \given \vx, \paramstar, \hyperparam) \in \R^{C \times C},
\end{equation}
which is positive semidefinite for common likelihoods used in classification and regression~\citep{murphy2012machine}.

The linearized Laplace approximation $\log q(\data \given \hyperparam)$ to the log marginal likelihood $\log p(\data \given \hyperparam)$ is then given by
\begin{equation}
\begin{aligned}
    \label{eq:marglik_la}
    \log q_\paramstar (\data &\given \hyperparam) = \log p(\data \given \paramstar, \hyperparam) + \log p(\paramstar \given \hyperparam) \\
    &- \half \log | \mJ_\paramstar\transpose \mLambda_\paramstar \mJ_\paramstar + \mP_0 | + \tfrac{P}{2} \log 2\pi,
\end{aligned}
\end{equation}
where $\mJ_\paramstar \in \R^{NC \times P}$ denotes stacked Jacobians $\mJ_\paramstar(\vx_n;\hyperparam)$, and $\mLambda_\paramstar \in \R^{NC \times NC}$ is a block-diagonal matrix with blocks $\mLambda_\paramstar(\vx_n;\hyperparam)$ for $n=1,..,N$, respectively, and we have suppressed the dependency on hyperparameters \hyperparam for notational brevity.
The matrix $\mP_0 \in \R^{P \times P}$ is the Hessian \wrt weights $\param$ of the log prior $\log p(\param \given \hyperparam)$ and is diagonal here. 
Because the first two terms of the approximation constitute the log joint, i.e., the training loss, all terms are simple to estimate and differentiate except for the (log-)determinant, which can be written equivalently~\citep{immer2021scalable} as
\begin{equation}
    \label{eq:log_det_kernel}
    | \mJ_\paramstar\transpose \mLambda_\paramstar \mJ_\paramstar + \mP_0 |
    =  | \mJ_\paramstar \mP_0\inv \mJ_\paramstar\transpose \mLambda_\paramstar + \mI| |\mP_0|,
\end{equation}
and allows us to either compute the log-determinant of a $P \times P$ or $NC \times NC$ matrix.
Unfortunately, the log-determinant prohibits unbiased stochastic gradients on batches of data and is therefore not amenable to optimization with SGD.

\subsection{Generalized Gauss-Newton and Neural Tangents}
\label{sec:ggn_and_ntk}

The two alternative ways to compute the determinant in \cref{eq:log_det_kernel} use either the generalized Gauss-Newton (\ggn) form or the neural tangent kernel (\ntk) form.
In particular, 
\begin{equation}
\label{eq:ggn_and_ntk}
    \mH \defeq \mJ_\paramstar\transpose \mLambda_\paramstar \mJ_\paramstar \,\,\, \textrm{and} \,\,\,
    \mK \defeq \mJ_\paramstar \mP_0\inv \mJ_\paramstar\transpose \mLambda_\paramstar
\end{equation}
are the \ggn approximation to the Hessian and the \ntk with additional scaling by $\mP_0\inv$ and $\mLambda_\paramstar$, respectively.\footnote{We recover the standard finite width \ntk for a Gaussian likelihood and prior with unit variance, i.e., $\mK= \mJ_\paramstar \mJ_\paramstar\transpose$.}
To obtain marginal likelihood gradients, we need to estimate and differentiate these two quantities.
However, they are intractable for deep neural networks due to large numbers of parameters $P$ and many data points and outputs $NC$.

The \ggn can be made tractable by structured block-diagonal and diagonal approximations~\citep{martens2015optimizing}.
Without loss of generality, we assume ordered parameter indices $\{1,..,P\}$. 
With a contiguous partition\footnote{A partition of a set $\sS$ is a set $\sP$ of disjoint subsets of $\sS$ whose union is equal to $\sS$. It is contiguous if the elements in $\sP$ are ordered and the order remains preserved in $\sS$.}
$\sP=\{\sP_s\}_{s=1}^S$ of parameter indices, we have
\begin{equation}
\label{eq:block_ggn}
    \mH \approx \mH_{\sP} =  
    \begin{bmatrix}
    \mH_{\sP_1} & \cdots & 0 \\
    \vdots & \ddots & \vdots \\
    0 & \cdots & \mH_{\sP_S}
    \end{bmatrix}
    = \mH_{\sP_1} \underbrace{\oplus ... \oplus}_{\mathclap{\textrm{block-diagonal stacking}}} \mH_{\sP_S},
\end{equation}
a block-diagonal approximation where each block corresponds to a subset $\sP_s$.
For example, if each $\sP_s$ corresponds to the parameters of one particular layer, we recover a block-diagonal layer-wise approximation like \kfac~\citep{martens2015optimizing}.
If each set only contains a single parameter index, $\sP_s = \{s\}$, we recover a diagonal approximation.

\subsection{Hyperparameter Optimization for Deep Learning}
\label{sec:dl_hparam_opt}
To optimize general hyperparameters, we estimate and differentiate Laplace approximations during neural network training using scalable approximations~\citep{immer2021scalable}.
The form of hyperparameters \hyperparam governs the complexity of obtaining their gradients.
For example, the typical problem of optimizing the prior precision, $\mP_0$, is relatively simple because it acts linearly on the \ggn and \ntk. However, general hyperparameters, such as invariances, affect the forward pass of the neural network and act non-linearly on \ggn and \ntk. Therefore, they require higher-order differentiation.

In the following, we illustrate the derived bounds for both settings, i.e., selecting a scalar or layer-wise prior precision (weight decay) on MNIST, and rotational invariance on rotated MNIST. 
We use a subset of $1000$ random digits and apply a small 3-layer convolutional neural network with linear output layer.
This setting ensures that we can compute the exact linearized Laplace in \cref{eq:marglik_la} as a reference.

\section{Structured Parametric Laplace Approximations are Lower Bounds}
\label{sec:bounds}
Practical linearized Laplace approximations of deep neural networks use structured \ggn approximations, such as diagonal and block-diagonal variants, to make the log-determinant computation tractable~\citep{immer2021scalable, daxberger2021laplace, antoran2021linearised}.
However, it is unclear to what extent these approximations are related to the exact linearized \lap.
The following Theorem shows that such approximations are in fact lower bounds and therefore justifies their use better.
This is similar to variational inference~\citep{blei2017variational}, where more restrictive approximations, such as a diagonal instead of a multivariate Gaussian, lead to more slack in the evidence lower bound. 

\begin{restatable}[Parametric structured lower bound]{theorem}{parametricthm}
\label{thm:parametric}
For any partitioning $\sP$ of the parameter indices $\{1,..,P\}$, the block-diagonal approximation of $\mH_\sP$ to $\mH$ results in a lower bound to the \lap log marginal likelihood (\cref{eq:marglik_la}), i.e.,
\begin{equation}
    \log q_\paramstar(\D \given \vh) \geq \log p(\D, \paramstar \given \vh) - \tfrac{1}{2} \log | \mH_{\sP} + \mP_0 | + c,
\end{equation}
where $c = \tfrac{P}{2} \log 2\pi$.
Further, a refinement\footnote{A refinement $\sP'$ of $\sP$ is a partition such that for all $\sP_s' \in \sP'$ there exists $\sP_s \in \sP$ such that $\sP_s' \subseteq \sP_s$.} of the partition $\sP$ to $\sP'$ will result in a lower bound with even more slack.
\end{restatable}

\begin{example}
The block-diagonal Laplace approximation that only considers correlations within layers, similar to the popular \kfac\footnote{Note that \kfac is a further approximation to the block-diagonal and a bound is not theoretically guaranteed. Nonetheless, it does seem to hold in practice.}-Laplace~\citep{ritter2018scalable, daxberger2021laplace}, is a lower bound to the linearized Laplace approximation.
A diagonal Laplace approximation that only considers marginal variances of parameters is a further lower bound of such block-diagonal approximation.
\end{example}

For a proof, refer to \cref{app:theory}.
The bound holds for commonly used block-diagonal approximations and is further agnostic to the type of Hessian approximation.
That is, it holds for the Hessian, \ggn, or (empirical) Fisher variants of the \lap alike.
In \cref{fig:parametric_bound}, we show the bound for different hyperparameter values $\hyperparam$ of the prior precision and rotational invariance in the illustrative setting described in \cref{sec:dl_hparam_opt}.
For the prior precision, only the diagonal approximation has a large slack and significantly different optimum, which is in line with previous empirical observations~\citep{mackay1995probable, daxberger2021laplace}.
Perhaps surprisingly, it can still be useful to select the right rotational invariance because its shape and thus optimum is in line with the better approximations.
\begin{figure}[t]
    \centering
    \includegraphics[width=\columnwidth]{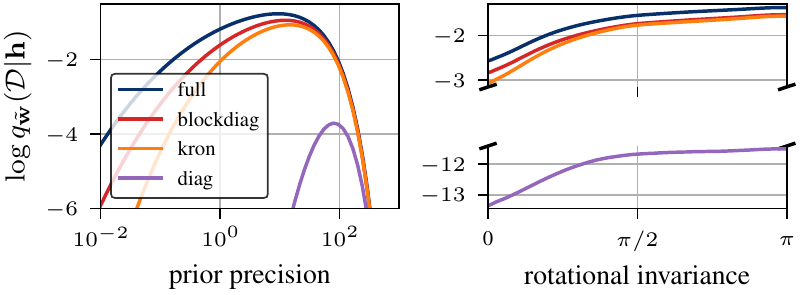}
    \vspace{-2em}
    \caption{Illustration of the parametric bound (\cref{thm:parametric}) with commonly used approximation structures for prior precision and rotational invariance as hyperparameters \hyperparam on MNIST with CNN.}
    \vspace{-1em}
    \label{fig:parametric_bound}
\end{figure}

While the above bounds justify existing Laplace approximations, they do not improve scalability with dataset size, because we have partitioned the parameters rather than the dataset. 
To do so, we employ the dual \ntk representation.

\section{Stochastic Gradients using the \ntk}
\label{sec:method}
While structured parametric lower bounds to the linearized Laplace marginal likelihood give good performance on model selection, a single hyperparameter gradient still requires an entire pass through the training data.
This is because the log-determinant is in general not separable across data points and computing it on a subset of data would lead to an uncontrolled upper bound to the marginal likelihood.
To overcome this limitation, we use the log-determinant in its dual \ntk form in \cref{eq:log_det_kernel} to devise lower bounds that enable batched computation on subsets of data and thus stochastic-gradient-based optimization.
Our bounds work for \ntk~(\cref{sec:kernel_bound}) or parametric and structured \ggn~(\cref{sec:parametric_sod_bound}) estimators alike and enable tightening the lower bound using larger batches of data or improving the partitioning of data points into batches~(\cref{sec:partitioning}).

\subsection{Subset-of-Data Kernel Bound}
\label{sec:kernel_bound}
Using the \ntk form of the log-determinant, we construct lower bounds for subsets of data that enable stochastic marginal likelihood gradients and can therefore greatly improve the efficiency of hyperparameter optimization.
Due to the cubic scaling with the number of data points, the \ntk form has so far been only used in special cases~\citep{immer2021scalable, antoran2022probabilistic}.
Using our bounds, however, the \ntk form on subsets of data becomes a tractable alternative for parametric variants like \kfac.
Further, computing the \ntk can be more architecture-agnostic~\citep{novak2019neural, novak2022fast} because it relies on plain Jacobians while structured approximations like \kfac are often non-trivial to compute or extend to custom architectures~\citep{dangel2019backpack, osawa2021asdfghjkl, kfac-jax2022github}.

Instead of partitioning the parameters as in \cref{thm:parametric}, we partition the $N$ inputs and $C$ outputs and make use of the \ntk form, which naively requires computing a $NC \times NC$ matrix, to obtain a lower bound for subsets of data. 
In particular, we have the set of indices $\sI = \{(n, c) \mid 1 \leq n \leq N, 1 \leq c \leq C\}$ corresponding to inputs $\vx_n$ and outputs $\evf_c$ of the neural network.
The index order is fixed \wrt $\mK$ but arbitrary.
With partitioning $\sB = \{\sB_m\}_{m=1}^M$ of $\sI$, we have
\begin{equation}
    \label{eq:block_kernel}
    \mK \approx \mK_{\sB} = \mK_{\sB_1} \oplus \cdots \oplus \mK_{\sB_M},
\end{equation}
where each $\mK_{\sB_m}$ is a \ntk on a subset of data and outputs.
Mathematically, we have $\mK_{\sB_m} = \mJ_{\sB_m} \mP_0\inv \mJ_{\sB_m}\transpose \mLambda_{\sB_m}$ where $\mJ_{\sB_m} \in \R^{|\sB_m| \times P}$ are the Jacobians for the input-output pairs and $\mLambda_{\sB_m} \in \R^{|\sB_m| \times |\sB_m|}$ the corresponding log-likelihood Hessian terms. 

\begin{restatable}[Data subset lower bound]{theorem}{kernelthm}
\label{thm:kernel}
For any partitioning $\sB$ of inputs and outputs $\sI$, the corresponding block-diagonal \ntk approximation $\mK_{\sB}$ to $\mK$ results in a lower bound to the linearized Laplace marginal likelihood, i.e.,
\begin{equation}
\begin{aligned}
\label{eq:kernel_bound}
    &\log q_\paramstar(\D \given \vh) \geq \log p(\D, \paramstar \given \hyperparam) - \tfrac{1}{2} \log |\mK_{\sB}+\mI| |\mP_0| + c \\
    &\propto \log p(\D, \paramstar \given \hyperparam) - \tfrac{1}{2} \log | \mP_0 | 
    - \tfrac{1}{2} \textstyle{\sum_{m=1}^M} \log |\mK_{\sB_m} + \mI|,
\end{aligned}
\end{equation}
where $c = \tfrac{P}{2}\log 2\pi$ is left out in the second line indicated by the proportionality.
Further, a refinement of partition $\sB$ to $\sB'$ will result in a lower bound with more slack.
\end{restatable}

\begin{figure}[t]
    \centering
    \includegraphics[width=\columnwidth]{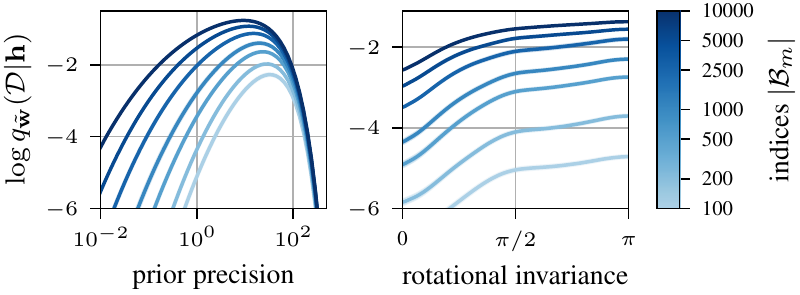}
    \vspace{-2em}
    \caption{Illustration of the subset-of-data \ntk lower bound~(\cref{thm:kernel}) with varying subset sizes for prior precision and rotational invariance as hyperparameters \hyperparam on the two MNIST variants.}
    \label{fig:ntk_bound_grid}
\end{figure}

We give a proof in \cref{app:theory}.
\cref{thm:kernel} shows that we can use the \ntk on subsets of data and outputs (for multi-output cases like classification) and obtain a valid lower bound.
Further, the form of \cref{eq:kernel_bound} allows unbiased stochastic estimation, gradients, and therefore SGD-based optimization of hyperparameters using (mini-)batches of data.
That is, by uniformly sampling index sets $\sB_m$, we have
\begin{equation}
\label{eq:stochastic_estimation}
    \textstyle{\sum_{m=1}^M} \log | \mK_{\sB_m} + \mI | = M \E_{m \sim \sU[M]} [\log |\mK_{\sB_m} + \mI|].
\end{equation}
In \cref{fig:ntk_bound_grid}, we show the tightness of the bound for different subset sizes $|\sB_m|$ on the two illustrative problems.
Despite small subset sizes, the bounds are tighter than the crude diagonal approximation and already around $2$\% of the input-output pairs can suffice to select good hyperparameters.

\subsection{Subset-of-Data Parametric Bounds}

\label{sec:parametric_sod_bound}
Computing the \ntk on subsets of data can be efficient but is non-trivial to parallelize, in which case parametric approximations like \kfac might be preferable~\citep{osawa2019large}.
Using the \ntk lower bound in \cref{thm:kernel}, we can devise an equivalent parametric estimator, which gives a lower bound on subsets of data (and parameters).
This enables the use of approximations like \kfac with Laplace on batches of data and thus for SGD-based hyperparameter optimization.

Defining $\mH_{\sB_m} \defeq \mH_{\sB_m}=\mJ_{\sB_m}\transpose \mLambda_{\sB_m} \mJ_{\sB_m}$ as a form of \ggn on the subset of inputs and outputs $\sB_m$, we have equivalence to the \ntk bound, which we show in \cref{app:theory}:

\begin{restatable}[Parametric data subset bound]{theorem}{paramthm}
\label{thm:parametric_sod_bound}
The \ntk-based lower bound in \cref{eq:kernel_bound} to the linearized Laplace marginal likelihood is equivalent to a parametric variant,
\begin{align}
\label{eq:parametric_sod}
    &\log p(\D, \paramstar \given \hyperparam)\!-\! \tfrac{1}{2} \log | \mP_0 | 
    \!-\! \tfrac{1}{2} \textstyle{\sum_{m=1}^M} \log |\mK_{\sB_m}\!\!+\!\mI| \\
    \!&=\!\log p(\D, \paramstar \given \hyperparam)\!+\!\tfrac{M\sminus1}{2} \log |\mP_0 |
    \sminus \half \textstyle{\sum_{m=1}^M} \log |\mH_{\sB_m}\!\!\!+\!\mP_0|, \nonumber
\end{align}
where $\mH_{\sB_m}=\mJ_{\sB_m}\transpose \mLambda_{\sB_m} \mJ_{\sB_m}$ is estimated on a data subset. 
\end{restatable}

This Theorem is useful because it shows how to do stochastic estimation of the marginal likelihood on a subset of data using a parametric form, which resembles the \ggn, instead of the \ntk.
However, the full \ggn is quadratic in the number of parameters and cannot be estimated in deep learning settings.
To overcome this, we can further combine it with the parametric lower bound~(\cref{thm:parametric}) to justify structured parametric approximations like the block-diagonal and its scalable approximation, \kfac, on subsets of data.

\begin{restatable}[Parametric doubly lower bound]{corollary}{doublycor}
\label{cor:doubly_bound}
For any partitioning $\sB$ of input-output pairs $\sI$ and $\sP$ of parameter indices $\{1, .., P\}$, respectively, the following lower bound to the linearized Laplace marginal likelihood holds:
\begin{align}
\label{eq:parametric_double_bound}
    \log q_\paramstar(\D \given \hyperparam) \geq &\log p(\D, \paramstar \given \hyperparam) + \tfrac{M-1}{2} \log | \mP_0 | \\
    &- \half \textstyle{\sum_{m=1}^M} \log | \mH_{\sB_m,\sP} + \mP_0| + \tfrac{P}{2} \log 2\pi, \nonumber
\end{align}
where $\mH_{\sB_m,\sP}$ is a block-structured \ggn approximation on subsets of input-output pairs.
\end{restatable}

\begin{figure}[t]
    \centering
    \includegraphics[width=\columnwidth]{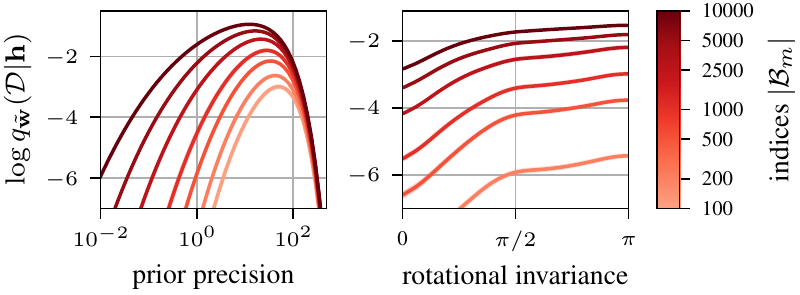}
    \vspace{-2em}
    \caption{Illustration of the parametric doubly lower bound using the block-diagonal \ggn~(\cref{cor:doubly_bound}) with varying subset sizes for prior and invariance hyperparameters. Setup as in \cref{fig:ntk_bound_grid}.} 
    \label{fig:doubly_bound}
\end{figure}

This doubly lower bound shows how parametric estimators, the most frequently used ones, can also be used for stochastic estimation and optimization of the log marginal likelihood.
In \cref{fig:doubly_bound}, we show the tightness of such bounds for the block-diagonal \ggn.
Surprisingly, even with just $5\%$ of the input-output pairs, the bound is tighter then the diagonal approximation.
Further, it only incurs a slight increase in slack compared to its \ntk-based upper bound in \cref{fig:ntk_bound_grid}.
In our experiments, we find that \cref{cor:doubly_bound} successfully enables stochastic-gradient-based hyperparameter optimization with \kfac, which previously could only be used in the full-batch setting.

\subsection{Tighter Bounds with Better Partitions}
\label{sec:partitioning}

\begin{figure}[t]
    \centering
    \includegraphics[width=\columnwidth]{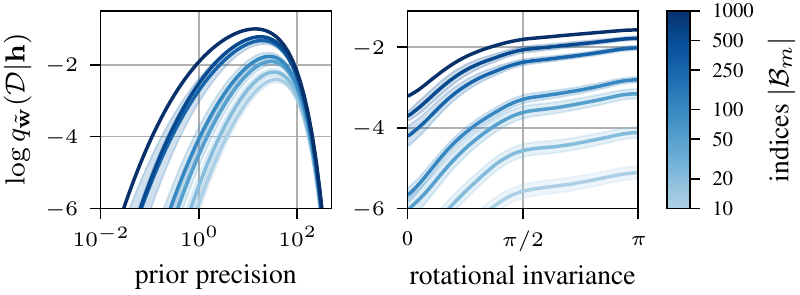}
    \vspace{-2em}
    \caption{\ntk Laplace bound using output-wise partitioning leads to tighter bounds and more efficiency. In comparison to \cref{fig:ntk_bound_grid}, the index set sizes are $C=10$ times smaller corresponding to an improved complexity by a factor of $C^3$ at similar tightness.}
    \vspace{-1em}
    \label{fig:classwise}
\end{figure}

The tightness of the lower bounds derived from the \ntk variant can be controlled with the partitioning $\sB$ into input-output pairs.
$\sB_m$ can be understood as batches like in parameter optimization using SGD.
The bounds derived from the \ntk motivate various choices of partitioning input data points and output dimensions so as to make the bounds as tight as possible.
Mathematically, we want to find the partition $\sB$ that maximizes the lower bound in \cref{eq:kernel_bound}.
We discuss two simple partitioning cases $\sB$ that can make the bounds derived from the \ntk potentially tighter: partitioning output dimensions and grouping inputs by labels.

Partitioning the \ntk by outputs means that each partition $\sB_m$ contains only one particular output and corresponds to an independent kernel as, for example, used in Gaussian processes~\citep{rasmussen2006gaussian}.
Each $\mK_{\sB_m}$ is then an output-wise \ntk, which can be much more efficient to compute than a full \ntk.
In the parametric space, this corresponds to a \ggn for a single output and also greatly improves efficiency as it makes the \ggn $C$ times cheaper and is theoretically justified through a bound.
\cref{fig:classwise} shows that the output-wise \ntk-based marginal likelihood bound is almost as tight as the full one at more than $C$ times smaller cost.
Already $0.2$\% of the input-output pairs, or equivalently $2$\% of the data, suffice to learn invariances. %
Across the entire design space of estimators, it is this class-wise partitioning that provides many Pareto-optimal estimators in \cref{fig:pareto}.

Alternatively, we can partition the index set $\sI$ using the label information we have for each $\vx_n$.
In classification, we have access to labels $y_n$ for each $\vx_n$.
Assuming that the correlation of the \ntk intra-class is greater than the anticorrelation inter-class, it is reasonable to partition $\sI$ in such a way that data points $\vx_n$ with the same labels are in the same subset(s).
In practice, this means that batches are made up from data points within the same class.
We find that this approach can perform well but incurs a larger variance~(cf.~\cref{app:fig:grouped} in \cref{app:illustration}). %
In future work, it could be interesting to develop methods that track (anti-)correlation between inputs in the \ntk to group them optimally during training. 
This is similar to inducing point optimization~\citep{titsias2009_variational_gp}.

\subsection{Family of Estimators and Algorithm}
Through the \ntk-based lower bound in \cref{thm:kernel}, we derived estimators that enable stochastic marginal likelihood estimates and gradients. 
This allows us to optimize hyperparameters with SGD, just like neural network parameters.
In particular, \cref{eq:stochastic_estimation} shows that we can obtain a stochastic unbiased estimator of the lower bounds on the log-determinants derived from the \ntk.
The remaining terms are the log likelihood, which allows for an unbiased stochastic estimate naively, and simple terms like the log determinant of the prior and constants.

\begin{algorithm}[tb]
   \caption{Stochastic Marginal Likelihood Estimate}
   \label{alg:sgd_marglik}
\begin{algorithmic}[1]
  \STATE {\bfseries Input:} dataset $\data$, likelihood, prior, random batch $\sB_m \in \sB$ (partition of input-output indices; $|\sB|=M$), structure (e.g. \ntk or \ggn), param partition $\sP$ (opt.).
    \STATE Let $\sD_m$ denote data points of input indices in $\sB_m$
    \STATE $\sL \gets \tfrac{|\data|}{|\sD_m|} \log p(\data_{m} \given \tilde{\param}, \hyperparam) + \log p(\tilde{\param} \given \hyperparam) - \half \log |\mP_0|$ 
    \IF{structure = \ntk}
        \STATE $\log q_{\tilde{\param}}(\data \given \hyperparam) \gets \sL -\tfrac{M}{2} \log | \mK_{\sB_m}\!+\!\mI |$
    \ELSE
        \STATE $\log q_{\tilde{\param}}(\data \given \hyperparam) \gets \sL - \tfrac{M}{2} \log  |\mH_{\sB_m, \sP} \mP_0\inv + \mI|$
    \ENDIF
 \STATE Return log marginal likelihood estimate $\log \tilde{q}{\paramstar}(\data \given \hyperparam)$.
\end{algorithmic}
\end{algorithm}

In practice, we use our estimators for interleaved optimization of neural network weights and hyperparameters as in \citep{immer2021scalable}.
That is, we take gradient-steps on the hyperparameters every $k$th epoch after an initial burnin of $b$ epochs.
To obtain hyperparameter gradients, we choose a partitioning of the input-output indices $\sI$ into $\sB$, which contains $M$ such batches, and then sample from these uniformly at random.
Further, we resample the partition every epoch.
In \cref{app:illustration}, we show trajectories of this algorithm corresponding to the setting used for illustrating the bounds.

To keep the bounds tight, we recommend output-wise partitioning since it eliminates complexity scaling in the number of outputs $C$ while empirically maintaining a relatively tight lower bound.
In practice, the most scalable bounds for hyperparameter gradients are the \ntk-based bounds and doubly bounds using an efficient parametric approximation like \kfac,
which even work well with relatively small subset sizes.
Interestingly, the output-wise approximation can also be applied to \kfac-\ggn giving a differentiable and scalable alternative for general hyperparameter learning.

\cref{alg:sgd_marglik} shows how to compute stochastic marginal likelihood estimates.
We use automatic differentiation to obtain gradients \wrt $\hyperparam$ to train them with SGD.
To make the bound as tight as possible and thus improve the approximations, we choose the partitioning and the size of its batches so as to fully utilize the available memory in practice.

\begin{table*}[th!]
\begin{center}
\begin{adjustbox}{max width=\textwidth}
\begin{small}
\begin{tabular}{llcccrcccr}
\toprule
&Dataset & \multicolumn{4}{c}{CIFAR-10} & \multicolumn{4}{c}{CIFAR-100} \\
\cmidrule(lr){3-6} \cmidrule(lr){7-10} 
\hyperparam & Estimator & $\log q_{\tilde{\param}}(\data \given \hyperparam)$ & log lik. & acc. [\%] & time & $\log q_{\tilde{\param}}(\data \given \hyperparam)$ & log lik. & acc. [\%] & time \\
\midrule
\multirow{6}{*}{\rotatebox[origin=c]{90}{prior precision}}&\textsc{baseline}       &     - & -1.22 ± 0.04 & 84.6 ± 0.3 &   35\% &     - & -3.35 ± 0.14 & 56.7 ± 1.1 &    8\% \\
&\ntk-500-1    &    -0.87 ± 0.02 & -0.53 ± 0.01 & 86.8  ± 0.1&   60\% &    \textbf{-1.98 ± 0.08} & -2.12 ± 0.04 & 61.7 + 0.6 &   16\% \\
&\kfac-500-1     &    -1.41 ± 0.02 & \textbf{-0.38 ± 0.00} & 88.2 ± 0.1 &   65\% &    -4.73 ± 0.11 & -1.31 ± 0.04 & 63.9 ± 0.9 &   24\% \\
&\kfac-500-10    &    -1.31 ± 0.01 & \textbf{-0.38 ± 0.00} & 88.6 ± 0.1 &   81\% &    -3.95 ± 0.04 & -1.26 ± 0.01 & \textbf{67.9 ± 0.2} &   70\% \\
&\kfac-$N$-1 &    -1.10 ± 0.01 & \textbf{-0.38 ± 0.00} & 88.8 ± 0.1 &   52\% &    -4.31 ± 0.09 & \textbf{-1.22 ± 0.02} & 66.8 ± 0.5 &   15\% \\
&\kfac-$N$-$C$  &    \textbf{-0.79 ± 0.01} & -0.39 ± 0.00 & \textbf{89.2 ± 0.1}  &  100\% &    -2.38 ± 0.03 & -1.65 ± 0.07 & 64.6 ± 0.4 &  100\% \\
\midrule
\multirow{4}{*}{\rotatebox[origin=c]{90}{invariance}}
&\ntk-150-1    &    -0.56 ± 0.05 & -0.42 ± 0.01 & 90.3 ± 0.2 &    9\%  & \textbf{-1.19 ± 0.07} & -1.45 ± 0.03 & 69.5 ± 0.8 & 4\% \\
&\kfac-150-1   &    -0.59 ± 0.08 & \textbf{-0.29 ± 0.00} & 92.5 ± 0.2 &   10\% & -2.22 ± 0.37 & -1.40 ± 0.05 & \textbf{71.8 ± 0.4} & 5\% \\
&\kfac-$N$-1 &    -0.44 ± 0.07 & -0.31 ± 0.01 & \textbf{93.0 ± 0.3} &   23\% & -3.43 ± 0.56 & \textbf{-1.20 ± 0.03} & \textbf{71.7 ± 0.7} & 10\% \\
&\kfac-$N$-$C$  &    \textbf{-0.24 ± 0.02} & -0.33 ± 0.01 & \textbf{93.2 ± 0.0} &  100\% & - & -& - & $\sim$100\% \\
\bottomrule
\end{tabular}
\end{small}
\end{adjustbox}
\end{center}
\vspace{-1em}
\caption{Benchmark of stochastic marginal likelihood estimators on CIFAR classification tasks with a Wide ResNet (16-4) learning layerwise prior precisions (top) and affine invariances (bottom). 
``500-1'' refers to a batch size of 500 and single-output approximation with $N$ and $C$ corresponding to the full-batch setting.
The timing is relative to the full-batch estimator, \kfac-$N$-$C$, which was previously used for these settings.
Our stochastic estimators can perform on par at up to 25-fold speed-up.
Due to the number of classes, \kfac-$N$-$C$ runs out of memory for CIFAR-100.
Performance bold per category if standard errors with best overlap from both sides.}
\label{tab:benchmark_combined}
\end{table*}

\section{Experiments}
\label{sec:experiments}
We experimentally validate the proposed estimators on various settings of marginal-likelihood-based hyperparameter optimization for deep learning.
Overall, we find that the lower bounds using subsets of data or outputs often provide a better trade-off between performance and computational or memory complexity.  
In particular, they remain relatively tight even when applied only on small subsets of data and outputs.
Therefore, they can greatly accelerate hyperparameter optimization with Laplace approximations, making marginal-likelihood optimization possible at larger scale.\footnote{Code: \href{https://github.com/AlexImmer/ntk-marglik}{\texttt{github.com/AlexImmer/ntk-marglik}}}

In our experiments, we optimize prior precision parameters, $\mP_0$, equivalent to weight-decay, per layer of neural networks~\citep{immer2021scalable, antoran2022probabilistic} and learn invariances from data~\citep{van2018learning, immer2022invariance}.
Learning invariances requires differentiating the \ntk or \ggn and therefore acts as a realistic example of gradient-based optimization for general hyperparameters, which should be the long-term goal of Bayesian model selection for neural networks.

In the following, we first discuss the illustrative example used throughout the theoretical development.
Further, we compare our estimators to the full dataset \kfac-Laplace approximation, which previously provided the best results for hyperparameter optimization with Laplace approximations~\citep{immer2021scalable, daxberger2021laplace, immer2022invariance}. 
In this case, we find that estimators based on subsets of data are significantly more efficient and can perform on par even when the dataset is large.
Lastly, we benchmark the two most scalable estimators on TinyImagenet, a setting, where full dataset Laplace approximations failed due to computational costs~\citep{mlodozeniec2023hyperparameter}.

\subsection{Tightness of Bounds, Performance, and Runtime}
\label{sec:illustration}
Throughout the theoretical development, we illustrate the tightness of the derived lower bounds to the linearized Laplace log marginal likelihood~(cf.~\cref{fig:parametric_bound}, \ref{fig:ntk_bound_grid}, \ref{fig:doubly_bound}, and \ref{fig:classwise}).
We use a small 3-layer convolutional network with linear output that has $P\approx 16\,000$ parameters on a random subset of $1000$ MNIST~\citep{lecun2010mnist} digits, which allows analytical computation of both the full \ggn and \ntk.
For the illustration using the rotational invariance, we additionally rotate digits by a random angle $\theta \sim \sU[-\pi, \pi]$. 
We compare the prior precision and rotational invariance hyperparameters with the log marginal likelihood estimate for the same trained neural network.
In \cref{app:illustration}, we additionally show the bound and corresponding test performance when optimizing it during training in both settings.

\paragraph{Parametric estimators.}
\cref{fig:parametric_bound} indicates that block-diagonal and \kfac variant attain almost the same marginal likelihood as a full \ggn and have the same test performance.
It also shows that the diagonal approximation can fail catastrophically due to its high slack in the bound~(cf.~\cref{thm:parametric}) as it is the most refined partitioning $\sP'$ of parameter indices.

\paragraph{Stochastic estimators.}
\cref{fig:ntk_bound_grid} illustrates the proposed subset-of-data marginal likelihood estimators showing that already a small fraction of the input-output pairs leads to tight lower bounds.
However, too small subset sizes can lead to failure when optimizing them, similar to the diagonal approximation.
The doubly lower bound, which allows using structured parametric estimators on subsets of data~(cf.~\cref{cor:doubly_bound}), is displayed in \cref{fig:doubly_bound} and is similarly tight. %
Further, handling the $C=10$ outputs independently for \ntk-based estimators, which corresponds to a partitioning by class, almost performs as well as a full \ntk in \cref{fig:classwise} and is often significantly cheaper.

\begin{figure*}[th!]
    \centering
    \includegraphics[width=\textwidth]{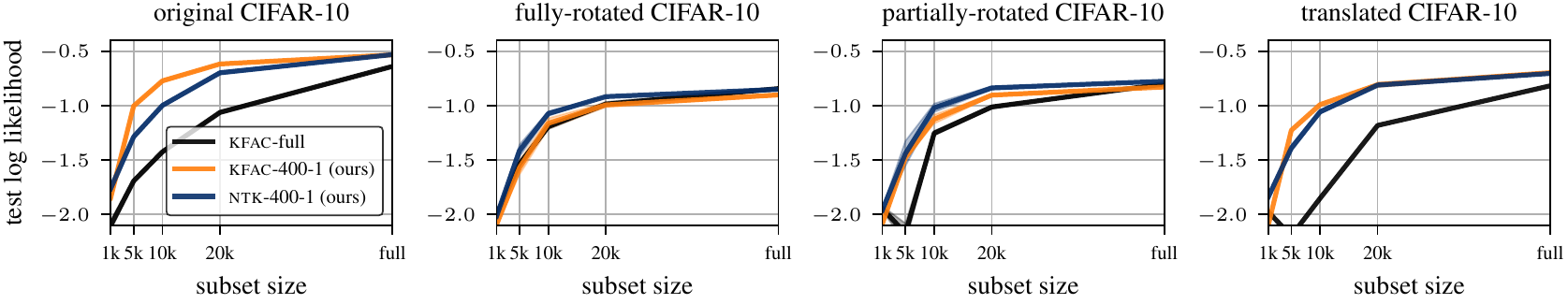}
    \vspace{-2em}
    \caption{Proposed stochastic marginal likelihood estimators using \ntk and \kfac with a fixed batch size of $400$ and class-wise partitioning perform well on subsets and modifications of CIFAR-10 with a ResNet.
    Our estimators can take more stochastic gradient steps than the full \kfac estimator at a faster runtime and thus greatly improve over it in terms of test log likelihood.
    This suggests that using many stochastic gradients can be more effective for hyperparameter learning than taking few exact gradients. More results in \cref{app:subsets}.}
    \label{fig:subsets}
\end{figure*}

\paragraph{Pareto-efficiency.}
\cref{fig:pareto} shows the obtained marginal likelihood estimates versus the runtime of a single estimate (and its gradient \wrt hyperparameters) indicating a Pareto-frontier between the two.
Previously, only the full-batch estimators (black markers) were known, of which only \kfac is both tractable and performant. 
Our derived estimators greatly increase the design space and provide many Pareto-optimal estimators, especially for lower runtime budgets.

\paragraph{Partitioning by class label.}
In \cref{sec:partitioning}, we hypothesized that grouping input-output sets by label information could improve bounds due to higher intra-class correlation than inter-class anti-correlation.
Across all runs, this did not make a significant difference for learning invariances.
However, for optimizing solely the prior precision~(cf.~\cref{app:illustration}) it improves the bound and test log likelihood by $10$\% on average.
This suggests that improved partitioning can help. %

\subsection{Benchmark of Proposed Estimators}
\label{sec:benchmark}

We compare the proposed marginal likelihood estimators to the \kfac-Laplace, the state-of-the-art among Laplace approximations for hyperparameter optimization~\citep{immer2021scalable, daxberger2021laplace, antoran2022probabilistic}. 
In the first setting, we optimize the layer-wise prior precision on a Wide ResNet 16-4~\citep{zagoruyko2016wide} on CIFAR-10~\citep{Cifar10} and CIFAR-100~\citep{cifar100}. 
In the second setting, we additionally optimize learnable invariances, similar to data augmentation, where the Laplace-based method by \citet{immer2022invariance} is extremely costly since it requires more than one pass through the dataset per gradient.
The results in \cref{tab:benchmark_combined} suggest that stochastic marginal likelihood estimators are particularly useful for learning invariances, where they accelerate the runtime up to $25$-fold at similar performance.
In addition to our stochastic estimators based on subsets of data, i.e., \ntk-150-1 and \kfac-150-1 with a batch size of 150 on single outputs, we show a full-dataset stochastic \kfac estimator that obtains a bound by sampling single random outputs~(cf. remark in ~\cref{app:theory}).

\subsection{Behavior with Varying Dataset Size}
Since larger datasets theoretically lead to looser bounds when using constant subsets of data, we investigate how this impacts performance in \cref{fig:subsets}.
In particular, we follow \citet{immer2022invariance} and learn invariances present in transformed versions of CIFAR-10~\citep{Cifar10} on random subsets ranging from $1\,000$ to all $50\,000$ data points.
We follow prior work by considering distribution over affine invariances~\citep{benton2020learning, van2021learning} detailed in \cref{app:subsets}.
We evaluate the test log likelihood of \kfac and \ntk-based estimators on subsets of 400 data points and single independent outputs, \kfac-400-1 and \ntk-400-1.
We compare the approach with the method by \citet{immer2022invariance}, \kfac-full (equal to \kfac-$N$-$C$).

The cheaper marginal likelihood estimates and gradients allow us to do more hyperparameter updates per epoch given equivalent computational constraints.
Compared to \kfac-full, we can therefore increase the number of hyperparameter steps for \kfac-400-1 and \ntk-400-1 from 10 to 100, while still reducing overall training time. 
\cref{fig:subsets} shows that our estimators perform on par or better than the full \kfac variant in terms of test log likelihood, especially, but not only, for small dataset sizes.
We hypothesize for invariance learning that the ability to do more gradient updates is more beneficial than having tighter lower bounds.
Further details on invariance learning, experimental details, and results can be found in \cref{app:subsets}.

\subsection{Scaling to Larger Datasets and Models}
For invariance learning, the results on CIFAR-100 in \cref{tab:benchmark_combined} already indicated that \kfac-$N$-$C$ is not scalable enough to enable gradient-based hyperparameter optimization.
\citet{mlodozeniec2023hyperparameter} also found that it is not possible to run it on the even larger TinyImagenet dataset~\citep{le2015tiny}, which has $N=100\,000$ data points and $C=200$ classes, using a ResNet-50 with roughly 23 million parameters.
Using the two fastest estimators for invariance learning, we show in \cref{tab:tiny} that our subset-of-data estimators enable optimizing hyperparameters in this setting.

We compare to the results provided by \citet{mlodozeniec2023hyperparameter} and use the same architecture but with Fixup~\citep{zhang2019fixup} instead of normalization layers, which would be incompatible with weight decay~\citep{antoran2021linearised}.
Optimizing layer-wise prior precisions can greatly improve over the baseline with default settings and no data augmentation.
Further, our estimators improve over invariance learning using neural network partitioning~\citep{mlodozeniec2023hyperparameter} and Augerino~\citep{benton2020learning}.

\subsection{Practical Considerations}
In our experiments, the proposed estimators using subset-of-data lower bounds using a single output often perform best, in particular using the parametric \kfac variant.
Although, the \ntk-based variant often yields a tighter bound, \kfac-based bounds seem particularly well-suited for marginal-likelihood optimization and are therefore preferable in practice.
As illustrated through invariance learning, our stochastic estimators are efficient for general differentiable hyperparameter optimization, which is an exciting future direction.

\begin{table}[t]
\begin{center}
\begin{adjustbox}{max width=\columnwidth}
\begin{small}
\begin{tabular}{llcc}
\toprule
\hyperparam & Estimator & log likelihood & accuracy [\%] \\
\midrule
\multirow{3}{*}{\rotatebox[origin=c]{90}{prior}} & \textsc{baseline} & -4.42 ± 0.08 & 45.6 ± 0.3 \\
&\ntk-400-1 & -4.16 ± 0.25 & 49.8 ± 0.2 \\
&\kfac-400-1 & \textbf{-2.17 ± 0.02} & \textbf{53.0 ± 0.1} \\
\midrule
\multirow{3}{*}{\rotatebox[origin=c]{90}{invariance}} &\textsc{augerino}
& - & 41.1 ± 0.2\\
& \textsc{partitioned}
& - & 48.6 ± 0.0 \\
& \ntk-70-1 & -3.15 ± 0.08 & 56.3 ± 0.1 \\
& \kfac-60-1 & \textbf{-2.02 ± 0.10} & \textbf{58.4 ± 0.3} \\
\bottomrule
\end{tabular}
\end{small}
\end{adjustbox}
\end{center}
\vspace{-1em}
\caption{Hyperparameter learning on TinyImagenet with ResNet-50 in comparison to Augerino~\citep{benton2020learning} and neural network partitioning taken from \citet{mlodozeniec2023hyperparameter}. While previous Laplace approximations are intractable in this setting, \kfac with a subset size of $60$ data points on single outputs excels.}
\label{tab:tiny}
\end{table}

\section{Related Work}
\label{sec:related_work}

\textbf{Marginal-likelihood optimization}, also referred to as \emph{empirical Bayes}, is the \emph{de-facto} standard for hyperparameter optimization of Gaussian process models~\citep{rasmussen2006gaussian} without validation data, and was used in the early days of Bayesian neural networks~\citep{mackay1992evidence, foresee1997gauss}. Use in modern larger neural networks dwindled, with \citet{blundell2015weight} reporting failure using mean-field variational approximations, although the approach continued to work in Deep Gaussian Processes~\citep{damianou2013deep,dutordoir2020dcgp}, and newer work demonstrates the feasibility in modern neural networks \citep{ober2021global,immer2021scalable}.
\citet{antoran2022sampling} recently proposed an interesting alternative based on sampling, which works for prior precisions but not general hyperparameters.
In the context of deep learning, the marginal likelihood has so far been used to select regularization strength, invariances, architectures, and representations. %
Our derived estimators further extend the family of Laplace approximations and provide alternatives to their parametric approximations.
The partitioning of the kernel in our estimators is similar to the Bayesian committee machine~\citep{tresp2000bayesian, deisenroth2015distributed}, which distributes Gaussian process inference into smaller kernels but does not give a valid lower bound on the marginal likelihood. 
For a detailed discussion on the benefits, but also issues, of the marginal likelihood, we refer to \citet[Sec.~7]{gelman1995bayesian}.

\textbf{\ntk and \ggn}
\citet{jacot2018neural} introduced the neural tangent kernel to characterize training dynamics of neural networks under squared loss.
The training dynamics are even available in a closed-form when considering infinite width.
However, computing kernels in this case is not trivial for common architectures~\citep{novak2019neural}.
Our estimators use the \ntk of neural networks at finite width, sometimes referred to as empirical \ntk~\citep{novak2022fast}.
Interestingly, the \ntk and Gauss-Newton approximation to the Hessian are dual to each other as shown by \citep{khan2019approximate}, and \citet{immer2020improving} for general likelihoods.
Our work builds on this duality and derives novel estimators from the \ntk viewpoint that allow stochastic estimation.
The \ntk is also used to estimate generalization without the marginal likelihood, e.g., in the context of neural architecture search~\citep{park2020towards, chen2021neural}.
Further, the linearized Laplace with \ntk been used to improve posterior predictives~\citep{deng2022accelerated, kim2022scale}.

\section{Conclusion}
\label{sec:conclusion}
In this paper, we have derived stochastic estimators for the linearized Laplace approximation to the marginal likelihood that are suitable objectives for stochastic-gradient based optimization of hyperparameters.
Our estimators are derived from a functional view of the Laplace using the neural tangent kernel and allow to trade off estimation accuracy and speed.
Our experiments show that the (mini-)batch estimators perform on par with previous full-batch estimators but are many times faster. %
This suggests that they could be useful for learning more complex hyperparameters with SGD.
Future research could further find ways to make bounds tighter, for example, by improving partitioning of the \ntk.
Further, it could be interesting to apply the fast estimators to large-scale Bayesian linear models. %

\section*{Acknowledgements}
A.I. gratefully acknowledges funding by the Max Planck ETH Center for Learning Systems (CLS).
The authors thank the reviewers for their constructive feedback and comments, in particular R2, who suggested Figures~\ref{fig:ntk_bound_grid}, \ref{fig:doubly_bound}, and \ref{fig:classwise}.

\bibliography{bibliography}
\bibliographystyle{icml2023}

\newpage
\appendix
\onecolumn
\section{Theoretical Results and Proofs}
\label{app:theory}
The following Theorem and corollary are required to bound the parametric and \ntk-based marginal likelihood approximations as they allow to bound the determinant, or the log-determinant, respectively by ignoring off-diagonal blocks of a matrix.
\begin{theorem}[Fischer's inequality~\citeyearpar{fischer1908det}]
\label{thm:fischer}
For a positive definite matrix $\mM$, as defined hereafter, it holds that
\begin{equation}
    \det \mM = \det \begin{bmatrix}
    \mA & \mB \\
    \mB\transpose & \mC 
    \end{bmatrix}
    \leq 
    \det\begin{bmatrix}
    \mA & \vzero \\
    \vzero & \mC 
    \end{bmatrix}
    = \det (\mA \oplus \mC)
    = \det{\mA} \det{\mC}.
\end{equation}
\end{theorem}

This Theorem is from \citet{fischer1908det}.
We refer to \citet{horn2012matrix} for a proof.
We immediately have the following useful corollary negative log-determinants, which show up in the Laplace approximation to the log marginal likelihood.
The result simply follows from the fact that the $\log$ is monotonically increasing.

\begin{corollary}
\label{cor:fischer_logdet}
For a positive definite matrix $\mM$ as defined in \cref{thm:fischer}, the following inequality holds
\begin{equation}
    -\log \det \mM =
    -\log\det\begin{bmatrix}
    \mA & \mB \\
    \mB & \mC 
    \end{bmatrix}
    \geq 
    -\log\det{\mA} -\log\det{\mC}.
\end{equation}
\end{corollary}

For our bounds, a more general partitioned form of $\mM$ is required, which corresponds to repeated bounding with \cref{thm:fischer} through partitions.
In particular, we deal with a matrix $\mM \in \R^{M \times M}$ and use indices $1$ through $M$, i.e., $[M] = \{1, \ldots, M\}$ that, for example, allow to denote the $m$th diagonal element as $\emM_{m,m}$.
Further, we introduced a notation based on index sets in \cref{sec:background} that allows to denote blocks or off-diagonal elements of $\mM$.
In particular, let $\sA \subset [M]$ and $\sA' \subset [M]$ disjoint subsets, $\sA \cap \sA' = \emptyset$ of the indices denoting dimensions in $\mM$, then we write $\mM_{\sA, \sA} = \mM_{\sA} \in \R^{|\sA| \times |\sA|}$ for the square block matrix with entries $\emM_{i,j}$ such that $i \in \sA$ and $j \in \sA$.
In line with this, off-diagonal blocks are $\mM_{\sA, \sA'} \in \R^{|\sA| \times |\sA'|}$.

Partitioning the dimension indices $[M]$ of $\mM$ into two disjoint subsets $\sA$ and $\sA'$, we have three blocks, $\mM_{\sA}, \mM_{\sA'}$, and $\mM_{\sA, \sA'}$. 
Since we can re-order the matrix to conform to such blocks, because it relies on simultaneous permutation of rows and columns, we can apply \cref{thm:fischer} to $\mM$ and have
\begin{equation}
    \det \mM = \det \begin{bmatrix}
    \mM_{\sA} & \mM_{\sA, \sA'} \\
    \mM_{\sA', \sA} & \mM_{\sA'} 
    \end{bmatrix}
    \leq 
    \det\begin{bmatrix}
    \mM_{\sA} & \vzero \\
    \vzero & \mM_{\sA'}
    \end{bmatrix}
    = \det (\mM_{\sA} \oplus \mM_{\sA'})
    = \det \mM_{\sA} + \det \mM_{\sA'}.
\end{equation}
Note that the first equality here only holds for the determinant and not for its matrix argument since the partitioning by indices permutes the rows and columns simultaneously, i.e., an even amount of times.
With this, we have the following Lemma

\begin{lemma}
\label{lem:partition_fischer}
For any partitioning $\sP=\{ \sP_s \}_{s=1}^S$ of dimension indices $[M]$ of the positive definite matrix $\mM$, we have
\begin{equation}
\label{app:eq:block_ggn}
    \det \mM = \det
    \begin{bmatrix}
    \mM_{\sP_1} & \mM_{\sP_1, \sP_2}  &\cdots & \mM_{\sP_1, \sP_S} \\
    \mM_{\sP_2, \sP_1} & \mM_{\sP_2, \sP_2} & \cdots & \vdots\\
    \vdots & \vdots & \ddots & \vdots \\
    \mM_{\sP_S, \sP_1} & \cdots & \cdots & \mM_{\sP_S}
    \end{bmatrix}
    \leq \det (\mM_{\sP_1} \oplus \cdots \oplus \mM_{\sP_S})
    = \prod_{s=1}^S \det \mM_{\sP_s},
\end{equation}
and further \emph{refining} the partitioning $\sP$ into $\sP'$, such that each element in $\sP'$ is a subset of $\sP'$ is a futher upper bound. 
\end{lemma}
\begin{proof}
To obtain this bound, we iteratively apply \cref{thm:fischer}. 
In particular, let $\sR_s$ be the complement of the first $s$ subsets, i.e., $\sR_s = \{\sP_1, \ldots, \sP_s\}^\complement$, for example, $\sR_1 = \sP_1^\complement$ and $\sR_S = \emptyset$.\footnote{The order $1, \ldots, S$ is arbitrary.}
The complement $(\cdot)^\complement$ here is with respect to the full set of indices $[M]$.
Using \cref{thm:fischer} $S$ times on these complementing sets, we have
\begin{equation*}
\begin{aligned}
    \det \mM &= \det \begin{bmatrix}
    \mM_{\sP_1} & \mM_{\sP_1, \sR_1} \\
    \mM_{\sR_1, \sP_1} & \mM_{\sR_1}
    \end{bmatrix}
    \leq \det \mM_{\sP_1} \det \mM_{\sR_1}
    = \det \mM_{\sP_1} \det \begin{bmatrix}
    \mM_{\sP_2} & \mM_{\sP_2, \sR_2} \\
    \mM_{\sR_2, \sP_2} & \mM_{\sR_2}
    \end{bmatrix}\\
    &\leq \det \mM_{\sP_1} \det \mM_{\sP_2} \det \mM_{\sR_2} 
    = \cdots 
    \leq \prod_{s=1}^S \det \mM_{\sP_s}.
\end{aligned}
\end{equation*}
This proves the main statement of the Lemma. 
By refining the partition, we simply extend the upper bounds since a refinement splits up each element in $\sP$ into at least one subset, which again enables application of \cref{thm:fischer}.
\end{proof}

\parametricthm*
\begin{proof}
Since the log-likelihood and log-prior terms are identical, we only need to inspect the relationship between $-\log |\mH + \mP_0|$ and $-\log |\mH_{\sP} + \mP_0|$.
Due to the assumption that $\mP_0$ is diagonal, it can be added to the individual blocks of $\mH_{\sP}$ since a diagonal is the most \emph{refined} partition possible on the indices $\{1, \ldots, P\}$.
\footnote{In case, $\mP_0$ would not be diagonal and its structure would not correspond to a refinement of $\sP$, one would have to treat it jointly with $\mH$ and apply the bound to $\hat{\mH} \defeq \mH + \mP_0$ and then $\hat{\mH}_{\sP}$.}
The lower bound on the negative log determinant and the increased slack when refining then follows from \cref{lem:partition_fischer}, which upper bounds the product of determinants, and thus lower bounds the negative log determinant~(\cref{cor:fischer_logdet}).
\end{proof}

\kernelthm*
\begin{proof}
As shown by \citet{immer2021scalable} and in \cref{sec:background}, the linearized Laplace approximation can be written in \ntk-form using the matrix determinant Lemma giving 
\begin{align*}
    \log q_\paramstar (\data \given \hyperparam) &= \log p(\data, \paramstar \given \hyperparam) - \half \log | \mJ_\paramstar \mP_0\inv \mJ_\paramstar\transpose \mLambda_\paramstar + \mI| |\mP_0| + \tfrac{P}{2} \log 2\pi \\
    &= \log q(\data, \paramstar \given \hyperparam) - \half \log |\mP_0 | + \tfrac{P}{2} \log 2\pi 
    - \half \log | \mK + \mI| \\
    &\geq \log q(\data, \paramstar \given \hyperparam) - \half \log |\mP_0 | + \tfrac{P}{2} \log 2\pi - \half \log | \mK_{\sB} + \mI | \\
    &= \log q(\data, \paramstar \given \hyperparam) - \half \log |\mP_0 | + \tfrac{P}{2} \log 2\pi - \half {\textstyle \sum_{m=1}^M} \log | \mK_{\sB_m} + \mI |,
\end{align*}
where we first re-order and use our definition of the scaled \ntk~(cf. \cref{sec:background}).
Then, we use the lower bound according to \cref{lem:partition_fischer} further giving us the statement that refinement leads to more slack. 
We again note that adding a diagonal, in this case $\mI$ to $\mK$, can simply be absorbed into the block-matrices and extracted afterwards when applying it.
In particular, define $\hat{\mK}=\mK + \mI$, apply the bound, and we obtain $\hat{\mK}_{\sB_m} = \mK_{\sB_m} + \mI$ for all $m$ and thus $\hat{\mK}_{\sB}=\mK_{\sB} + \mI$.
\end{proof}

\paramthm*
\begin{proof}
The idea is to apply the matrix determinant Lemma to move from the subset of data \ntk bound back to a parametric variant. 
This gives us a log determinant that depends on the \ggn defined on subsets of input-output pairs.
We subtract the log joint, $\log p(\data, \paramstar \given \hyperparam)$, and add $\half \log |\mP_0|$ from to both sides of \cref{eq:parametric_sod} and multiply by $2$ to abbreviate the following equations. 
We then have
\begin{align*}
    -\textstyle{\sum_{m=1}^M} \log |\mK_{\sB_m}+\mI|
    &= -\textstyle{\sum_{m=1}^M} \log |\mJ_{\sB_m} \mP_0\inv \mJ_{\sB_m}\transpose \mLambda_{\sB_m} + \mI|\\
    &= M \log | \mP_0 | - \textstyle{\sum_{m=1}^M} \log |\mJ_{\sB_m} \mP_0\inv \mJ_{\sB_m}\transpose \mLambda_{\sB_m} + \mI| |\mP_0 | \\
    &= M \log | \mP_0 | - \textstyle{\sum_{m=1}^M} \log |\mJ_{\sB_m}\transpose \mLambda_{\sB_m} \mJ_{\sB_m} + \mP_0 | \\
    &= M \log | \mP_0 | - \textstyle{\sum_{m=1}^M} \log |\mH_{\sB_m} + \mP_0 |,
\end{align*}
where $\mJ_{\sB_m}$ corresponds to the Jacobians only for the input-output pairs in $\sB_m$ and $\mLambda_{\sB_m}$ to the log-likelihood Hessian for these outputs~(cf.~\cref{sec:ggn_and_ntk}).
subtracting and adding the removed quantities again concludes the proof. 
\end{proof}

\paragraph{Remark.} The form of \cref{eq:parametric_sod}, in particular $\mH_{\sB_m}$ is reminiscent of a \ggn approximation but just on a subset of input-output pairs.
In particular, it depends on the subset structure chosen for the \ntk bound, and could be a class-wise \ggn or simply on a random subset of data.
Subsets of data do not change the shape of the parametric estimator, i.e., it still requires the log determinant of a $P \times P$ matrix but computing that matrix can be greatly sped up:
apart from the requirement to compute Jacobians for all $N$ data points, a major issue with the \ggn approximation is that its cost scales linearly with the outputs $C$ as well~\citep{botev2017practical}.
Therefore, it is often approximated using sampling via equality to the Fisher information~\citep{martens2015optimizing} for exponential family likelihoods of natural form~\citep{martens2014new}.
By partitioning the input-output pairs $[NC]$ into $C$ partitions, we can compute the \ggn for a single output and use it as a proper lower-bound to the \lap marginal likelihood at the same cost as the Fisher, which does not lead a lower bound of the \lap due to the required sampling approximation.
Further, we can apply our parametric bound in \cref{thm:parametric} to use mixed bounds on subsets of data and subsets of parameter groups enabling, for example, a lower bound to the \lap marginal likelihood using \kfac on a subset of data.

\doublycor*
\begin{proof}
This corollary simply follows from applying the parametric lower bound, \cref{thm:parametric}, to the parametric estimator on a subset of data in \cref{thm:parametric_sod_bound}.
\end{proof}

\section{Computational Considerations and Complexity}
The computational complexity of the proposed estimators as well as the \ggn, \ntk, and their corresponding approximations like \kfac greatly depend on the model architecture.
For simplicity, we assume a neural network with $P$ parameters in $L$ linear hidden layers with each width of $D$, inputs $\vx \in \R^D$ of the same dimensionality, with output dimensionality $C$, i.e., we have $P = DC + \sum_{l\in [L]} D^2$, where $DC$ are the parameters of the linear output layer.
Further, we have $N$ data points.
The complexity of automatically differentiating the log determinant of the \ggn or \ntk is equivalent to computing the log determinant of these matrices defined in \cref{eq:ggn_and_ntk}.
Computing the log determinant of the \ggn first requires computing the \ggn itself, which is \order{NP^2C} for computing the sum of $N$ Jacobian outer products, where each Jacobian is $C \times P$.
The log determinant is then additionally \order{P^3}.
In comparison, for naive computation of the \ntk log determinant, we compute the $NC \times NC$ kernel in \order{N^2C^2P} and its determinant in \order{N^3C^3}.
The complexity for \kfac-\ggn is \order{NLD^2C} for summing $N$ outer products for $L$ layers and its determinant \order{D^3} for each layer giving \order{LD^3}.
We have the following computational and complexities for computing and differentiating the log determinant using automatic differentiation:
\begin{equation}
    \ggn \in \order{NP^2C + P^3} \quad \ntk \in \order{N^2C^2P + N^3C^3} \quad \kfac\textrm{-}\ggn \in \order{NLD^2C + LD^3}.
\end{equation}
While these are for naive estimators that can be improved for certain architectures, e.g., see \citet{novak2022fast} for faster \ntk computations, these complexities describe the worst-case setting.

Our proposed lower bounds based on subsets of inputs and outputs can greatly improve these complexities. 
In particular, we can replace $N$ by a subset size $M \ll N$ and use output-wise bounds to obtain $C=1$.
In our experiments, the fastest and still performant methods are \ntk-$M$-$1$ and \kfac-$M$-$1$, which have a computational and memory complexity for estimation and automatic differentiation of \order{M^2P+M^3} and \order{MLD^2+LD^3}, respectively.
Depending on the architecture, the \ntk estimator can further be faster than \order{M^2P}.
Overall, the proposed methods attain an acceleration at least linear in the number of outputs $C$ and data points $N$. 
\newpage
\section{Additional Experimental Results and Details}
\label{app:experiments}
We provide additional experimental details and results complementing those in the main text here.
For our implementation, we modify and extend the \texttt{asdl} library~\citep{osawa2021asdfghjkl} that offers fast computation of \kfac and \ntk, as well as \texttt{laplace-torch}~\citep{daxberger2021laplace} for the marginal likelihood approximations.
The code is available at \url{https://github.com/AlexImmer/ntk-marglik}.

\subsection{Tightness of Bounds, Performance, and Runtime}
\label{app:illustration}
The experiments illustrating the bounds in the main text described in \cref{sec:illustration} are conducted on randomly chosen $1000$ MNIST subsets for three random seeds and further images are fully rotated at random, i.e., up to $\pm \pi$.
We then use the \textsc{lila} method proposed by \citet{immer2022invariance} to try to learn these underlying invariances, which is essential to generalize well on this task since standard convolutional and fully connected layers are not rotationally invariant.
We use a simple convolutional network with three layers, max pooling, and a linear classification head that totals roughly $16\,000$ parameters so that we can compute and differentiate the full \ggn and \ntk.
We use $30$ augmentation samples to average the outputs and optimize the network parameters, invariance parameters, and prior parameters with Adam~\citep{kingma2014adam}.
For network parameters we use a learning rate of $10^{-3}$ and decay it to $10^{-9}$ using cosine decay and use a batch size of $250$.
The invariance and prior learning hyperparameters follow the settings of \citet{immer2022invariance}: $10$ epochs burnin and then update hyperparameters every epoch with learning rates $0.1$ and $0.05$ for prior precision and invariance parameters, respectively. 
Both are decayed by a factor of $10$ using cosine decay.
We use an optimized network to assess the bounds in Figures~\ref{fig:ntk_bound_grid} to \ref{fig:classwise}.
After convergence, we apply the different bounds and assess them over the entire grid of rotational invariance $\eta \in [0, \pi]$.
Here, we additionally present results for hyperparameter optimization using the same bounds in \cref{app:fig:parametric_kernel_lila} and \cref{app:fig:doubly_classwise_lila}.
The experiments are run on an internal compute cluster with different NVIDIA GPUs.
The timing experiments are run on a single A100 sequentially to ensure comparability.

In addition to the invariance learning experiments, we compare the bounds for varying values of a scalar prior precision in all illustrative figures in the main text.
Here, we show results optimizing the layerwise prior precision in \cref{app:fig:parametric_kernel}, \ref{app:fig:doubly_classwise}, and \ref{app:fig:pareto_prior} again on $1000$ randomly sampled MNIST digits per seed but without rotating them.
When only learning prior precision parameters and not invariances, a slack in the bound translates more directly into a slack in the test performance as can be seen in the figures.
However, it is worth noting that the \ntk and \kfac estimators are significantly cheaper on large batches of data for just optimizing the prior precision, in comparison to learning invariances. 
In \cref{app:fig:pareto_prior}, we see that stochastic \ntk-based estimators are Pareto-optimal in many cases as it was the case for invariance learning.
We also find that small batch sizes do not lead to problems at larger scale~(cf.~\cref{sec:benchmark}) where they perform on par with full dataset marginal likelihood estimators.
In \cref{app:fig:grouped}, we show the effect of class-wise partitioning of batches on the bound in comparison to the random partition.
We find that the bound can become tighter for both prior precision and invariance learning but incurs a slightly higher variance as indicated by the standard error displayed.

\begin{figure}[bh!]
    \centering
    \includegraphics{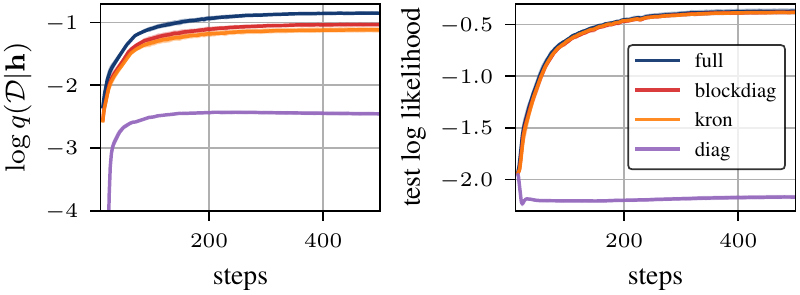}
    \hspace{1em}
    \includegraphics{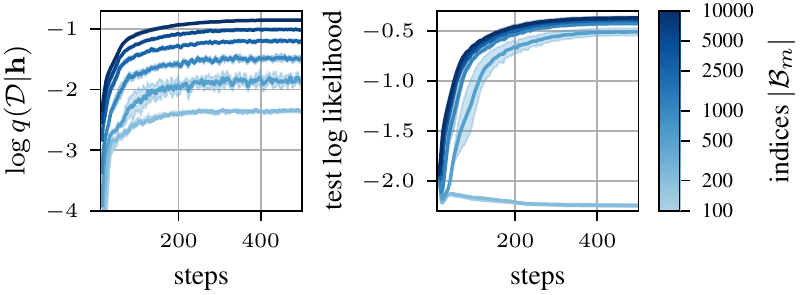}
    \vspace{-1em}
    \caption{Parametric and \ntk lower bound when optimizing invariances on random subsets of size $1000$ from rotated MNIST~\citep{lecun2010mnist}.
    Only small subset sizes and the diagonal approximation fail to learn the invariance and result in good test log likelihood.
    This figure corresponds to \cref{fig:parametric_bound} and \cref{fig:ntk_bound_grid} in the main text, which show the bounds constructed at step $500$ with parameters and hyperparameters trained by the ``full'' variant.}
    \label{app:fig:parametric_kernel_lila}
\end{figure}

\begin{figure}[th!]
    \centering
    \includegraphics{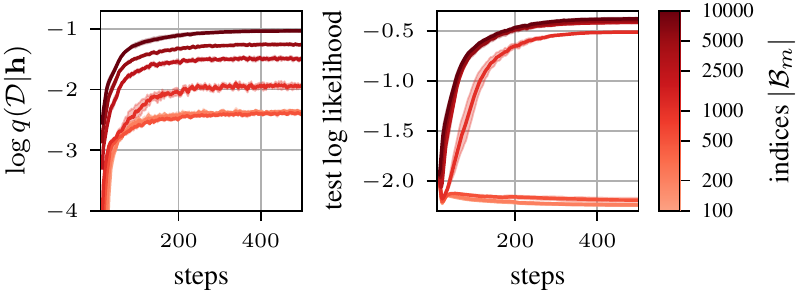}
    \hspace{1em}
    \includegraphics{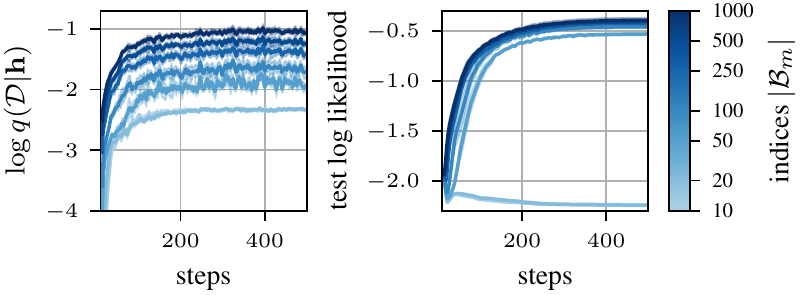}
    \vspace{-1em}
    \caption{Doubly parametric subset-of-data bound on the left with block-diagonal approximation and class-wise \ntk-based approximation. These figures correspond to \cref{fig:doubly_bound} and \cref{fig:classwise} in the main text but use the bounds for optimization of invariances during training. In line with the bounds, already small subset sizes suffice to pick up the invariance in rotated MNIST and achieve good test log likelihood.}
    \label{app:fig:doubly_classwise_lila}
\end{figure}

\begin{figure}[h!]
    \centering
    \includegraphics{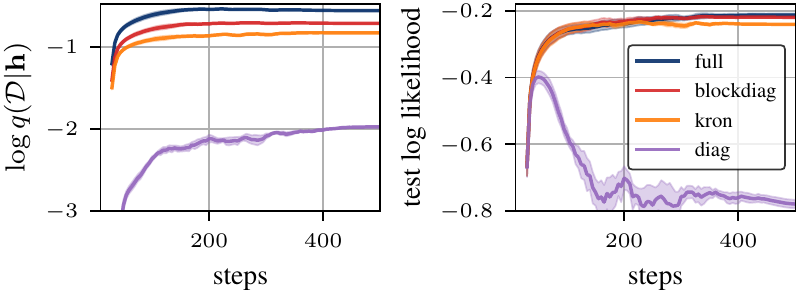}
    \hspace{1em}
    \includegraphics{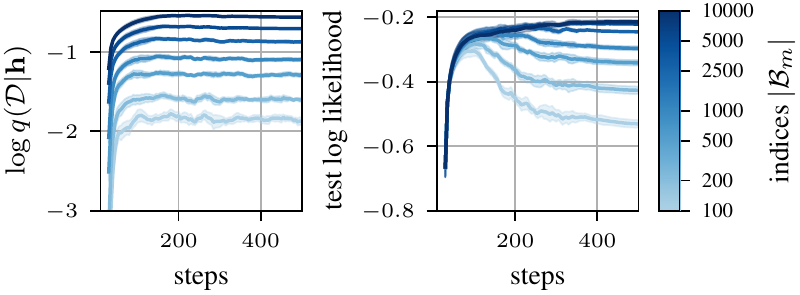}
    \vspace{-1em}
    \caption{Parametric and \ntk lower bound when optimizing prior precision and not invariances on random subsets of size $1000$ from MNIST~\citep{lecun2010mnist}.
    In contrast to invariance learning, lower bounds result in more reduction in test performance for the \ntk-based bounds that use stochastic estimates.
    This figure corresponds to \cref{fig:parametric_bound} and \cref{fig:ntk_bound_grid} in the main text.}
    \label{app:fig:parametric_kernel}
\end{figure}

\begin{figure}[h!]
    \centering
    \includegraphics{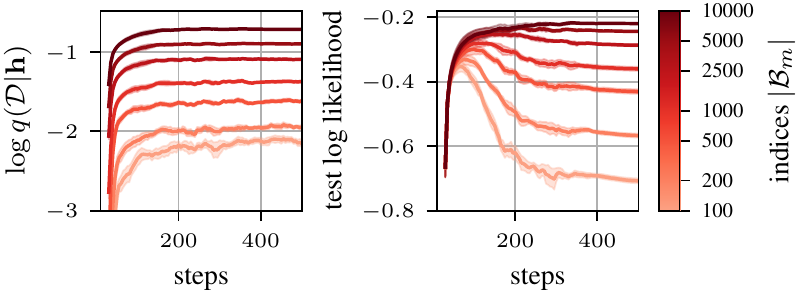}
    \hspace{1em}
    \includegraphics{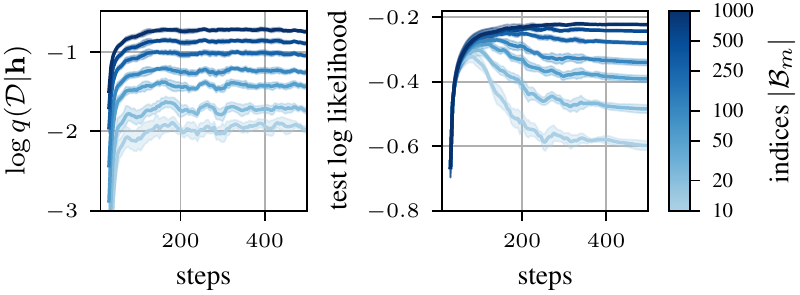}
    \vspace{-1em}
    \caption{Doubly parametric subset-of-data bound with block-diagonal approximation (left) and class-wise \ntk-based approximation (right) for prior precision optimization corresponding to \cref{fig:doubly_bound} and \cref{fig:classwise} in the main text.}
    \label{app:fig:doubly_classwise}
\end{figure}

\begin{figure}[th!]
    \centering
    \includegraphics{figures/grid_bound_kernel_so=False_grouped=False.pdf}
    \hspace{1em}
    \includegraphics{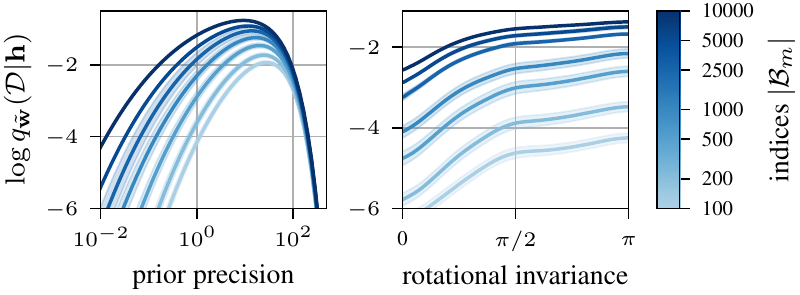}
    \vspace{-1em}
    \caption{Effect of partitioning data by class labels (right) in comparison to random partitioning (left; same figure as \cref{fig:ntk_bound_grid}) as proposed in \cref{sec:partitioning}. Using the class labels can slighty redue the slack of the bound but tends to incur a larger variance in the estimated bound as indicated by the shaded regions denoting one standard error.}
    \label{app:fig:grouped}
\end{figure}

\begin{figure}[h!]
    \centering
    \includegraphics{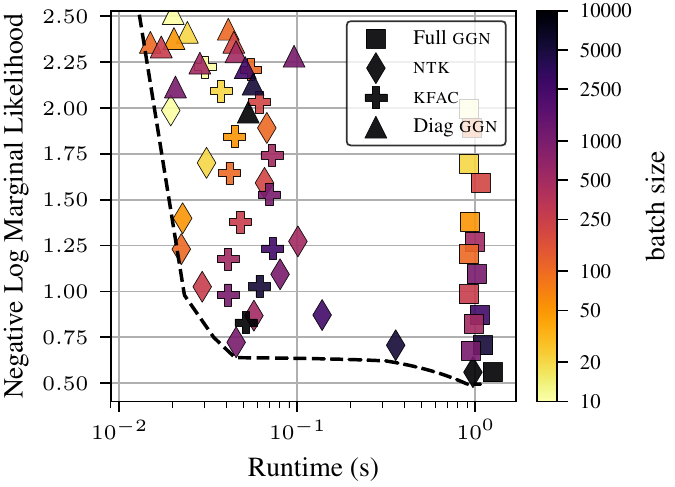}
    \vspace{-1em}
    \caption{Pareto-frontier for prior precision learning runs on random subsets of $1000$ MNIST digits with a small CNN.
    This figure corresponds to \cref{fig:pareto} in the main text, which displays estimators for additional invariance learning, not only learning priors.
    Many of the \ntk-based class-wise estimators are Pareto-optimal, i.e., the achieve the lowest marginal likelihood at a given runtime budget.}
    \label{app:fig:pareto_prior}
\end{figure}

\newpage
\subsection{Additional Details for Estimator Benchmark on CIFAR}
\label{app:benchmark}

Marginal likelihood estimation with the linearized Laplace in previous works was commonly applied to ResNets on CIFAR-10 and CIFAR-100, where it gives remarkable performance improvements over unregularized networks in the case where we have no data augmentation or prior information.
For example~\citet{immer2021scalable} and \citet{daxberger2021laplace} show that it improves the performance by $4\%$ points over the baseline on CIFAR-10.
Here, we reproduce this experiment and compare the performance of the previously used full-batch \kfac estimator, which produced the best results, to the proposed stochastic estimators.
The results are given in \cref{tab:benchmark_combined} in the top block.
Indeed, all estimators improve over the baseline that uses default settings of a Wide ResNet~\citep{zagoruyko2016wide} but without data augmentation and Fixup initialization instead of normalization layers~\citep{zhang2019fixup}. 
We use the Wide ResNet 16-4 with Fixup~\citep{zhang2019fixup}, since a Gaussian prior conflicts with normalization layers~\citep{zhang2019weight, antoran2021linearised}.
We optimize the network for $300$ epochs with a batch size of $128$ using SGD with momentum of $0.9$ and optimize the prior precision every $5$ epochs after a burnin phase of $10$ epochs for with $50$ gradient steps using a learning rate of $0.1$ and Adam~\citep{kingma2014adam} for both CIFAR-10 and CIFAR-100.

In the lower block of \cref{tab:benchmark_combined}, we use our estimators in the context of learning invariance using laplace approximations~\citep[\textsc{lila},][]{immer2022invariance}.
We use the default hyperparameters of their largest-scale experiment, where they learn invariances with Wide ResNets using $20$ augmentation samples.
Their method uses full dataset \kfac for $200$ epochs and takes more than $48$ hours on of training on CIFAR-10 with an NVIDIA A100 GPU because they have to compute a preconditioner on the entire dataset for each hyperparameter gradient. 
On CIFAR-100, due to the scaling with output classes, it is therefore intractable as it would likely take weeks. 
The default settings from CIFAR-10 lead to out-of-memory errors.
In contrast, our stochastic estimators can be directly evaluated on a small batch with a single output and thus improve the runtime by a factor $10$ while maintaining the same performance, i.e., compare $\kfac$-\textsc{150-1} with $\kfac$-$N$-$C$.
Also our single-output \kfac bound already increases runtime by a factor of $5$ and achieves the same performance.
The choice of the batch size in our bounds is such that the GPU memory is utilized as much as possible.
We therefore use NVIDIA A100 GPUs with 80GB memory to make the bounds as tight as possible and improve performance.
To compare the runtimes, all methods use the same GPU type and runtimes are averaged over 3 or 5 seeds for invariance and prior learning, respectively.

\subsection{Details on TinyImagenet Benchmark using ResNet-50}
We largely follow the setup of \citet{mlodozeniec2023hyperparameter} in this experiment and use their results on Augerino~\citep{benton2020learning} and their own network partitioning method for hyperparameter optimization.
The key difference is that we use Fixup~\citep{zhang2019fixup} instead of normalization layers in the ResNet-50 architecture.
Like \citet{mlodozeniec2023hyperparameter}, we use $16$ channels in the first layer and the standard growth factor of $2$ per layer.
The network has roughly $23\,000\,000$ parameters and therefore poses a complex task for differentiable hyperparameter optimization.
Even with 80GB of GPU memory, we can only fit $60$ to $70$ data points for differentiable \kfac or \ntk computations.
However, the performance is still sufficient to improve over the baselines.
We optimize the network weights for $100$ epochs and, as in the case for CIFAR, decay the learning rate with a cosine schedule from $0.1$ to $10^{-6}$.
The hyperparameters are optimized for $50$ steps every epoch after $10$ epochs of burn-in with a learning rate of $0.01$ decayed to $0.001$ with a cosine schedule.
The learning rate for prior parameters is constant at $0.1$.
For only prior precision optimization, we use the default setting by \citet{immer2021scalable} and update the prior parameters for $50$ steps every $5$ epochs after $10$ epochs of burn-in.

\clearpage
\newpage
\subsection{Additional and Detailed Results for Invariance Learning on CIFAR-10 Subsets.}
\label{app:subsets}

It is well known that soft or hard symmetry constraints are beneficial to machine learning methods and neural networks in particular, equivariant symmetries of (group-)convolutional layers being a canonical example \cite{cohen2016group}. Although explicitly embedding symmetry in architectures can be very effective, they need to be known and specified in advance and can not be adapted. The right symmetry can be difficult to choose and misspecified symmetries have the risk of restricting the capacity of a model to fit data \cite{wang2022approximately}. Symmetry discovery methods aim to automatically learn symmetries from available training data, which is a difficult task because symmetries provide constraints and are therefore not necessarily favorable in terms of typical training losses. In invariance learning, we parameterise a learnable distribution $\vg(\vx ; \augmentation)$ over possible transformations on inputs, defined by invariance parameters $\augmentation$. By averaging outcomes of a regular non-invariant neural network $\vf$ over inputs transformed by $\vg$, we obtain a predictor $\tilde{f}$, which is (soft-) invariant to the transformations specified by $\augmentation$. We might take $S$ Monte Carlo samples $\vepsilon_{1},\ldots,\vepsilon_{S} \simiid p(\vepsilon)$~\citep{van2018learning, benton2020learning} to obtain an unbiased estimate of the invariant function $\vfaug$:
\begin{align}
    \vfaug(\vx; \param, \augmentation) = \E_{p(\vx'|\vx,\augmentation)} [\vf(\vx'; \param)]
    = \E_{p(\vepsilon)} [ \vf(\vg(\vx,\vepsilon;\augmentation);\param)]
    \approx \tfrac{1}{S} {\textstyle \sum_{s}} \vf(\vg(\vx,\vepsilon_s;\augmentation);\param)
    \,,
    \label{eq:invariant_function}
\end{align}
where $\vfaug$ is differentiable in both model parameters $\param$ as well as invariance parameters $\augmentation$ through the reparameterisation trick \cite{kingma2013auto}. In our experiments, we consider a combination of uniform distributions on affine generator matrices resulting in a 6-dimensional invariance parameter $\augmentation \in \R^6$ that defines a density over the group of affine transformations, individually controlling x-translation, y-translation, rotation, x-scaling, y-scaling and shearing. Details on the used parameterisation can be found in \citet{benton2020learning} and App. B of \citet{immer2022invariance}.

Jointly learning model parameters \param and invariance parameters \augmentation is not a trivial exercise, as typical maximum likelihood solutions will always select the least restrictive no invariance in order to fit training data best, given a sufficiently flexible model. To overcome this issue, some people have considered the use of validation data, such as cross-validation or more sophisticated approaches, which may require expensive retraining or involved outer loops. Alternatively, Augerino \citep{benton2020learning} learns invariance solely on training data, but relies on explicit regularisation in order to do so, which requires knowledge of the used invariance parameterisation and additional tuning. For a discussion on particular failure cases of this approach we refer to Sec. 2. and App. C. of \citet{immer2022invariance}, which proposes to use scalable marginal likelihood estimates to learn invariance parameters. The marginal likelihood offers a principled way to learn invariances through Bayesian model selection, balancing data fit and model complexity through an Occam's razor effect. Unlike prior works that rely on the true marginal likelihood \cite{van2018learning} or variational lower bounds thereof \citep{van2018learning,van2021learning}, modern differentiable Laplace approximations of the marginal likelihood can scale to large datasets and deep neural networks, such as ResNets. We consider this as a baseline, and denote it by \kfac-full as it uses the full dataset to estimate the marginal likelihood, which leads to a significant slowdown in training compared to classical \map training. To improve estimates, we have derived lower bounds to the marginal likelihood that can be computed much more cheaply on subsets of the training data. This allows us to take more hyperparameter gradient steps per epoch at equivalent or faster training time. We hypothesize that taking more gradient steps can be beneficial, even though individual marginal likelihood estimates provide looser bounds.

We empirically evaluate performance of our method in conjunction with invariance learning on subsets of the same transformed versions of CIFAR-10 as used in \citet{immer2022invariance}. The use of reduced dataset sizes is a common data efficiency benchmark for invariance learning methods \cite{schwoebel2022layer}. We measure performance for varying total dataset sizes where we estimate marginal likelihoods on fixed subsets of 400 datapoints and use independent outputs, \kfac-400-1 and \ntk-400-1. We compare to \kfac-full which always uses the full set of available training data points for its estimates. We use a ResNet architecture with Fixup \cite{zhang2019fixup}, to prevent Gaussian prior conflicts with normalization layers \citep{zhang2019weight,antoran2021linearised}, and take $S{=}20$ Monte Carlo of invariance transformations. We optimize the network for 200 epochs with a batch size of 128 using Adam \cite{kingma2014adam} with cosine annealed learning rates starting at 0.1 for parameters and prior variances and invariance parameter learning rates starting at 0.05 for \kfac-full and 0.01 for \kfac-400-1 and \kfac-400-1, which use more gradient steps.

In the figures below, we show the marginal likelihood estimate (\cref{fig:subsets_marglik}), test log-likelihood (\cref{fig:subsets_nll}) and test accuracy (\cref{fig:subsets_perf}). We find that marginal likelihood estimates of \kfac-full are higher, which is not surprising since \kfac-400-1 and \kfac-400-1 use looser bounds. In terms of test performance, we find that using more approximate gradients can at the same time improve runtime as well as test accuracy and test log-likelihood.

\begin{figure*}[h]
    \centering
    \includegraphics[width=\textwidth]{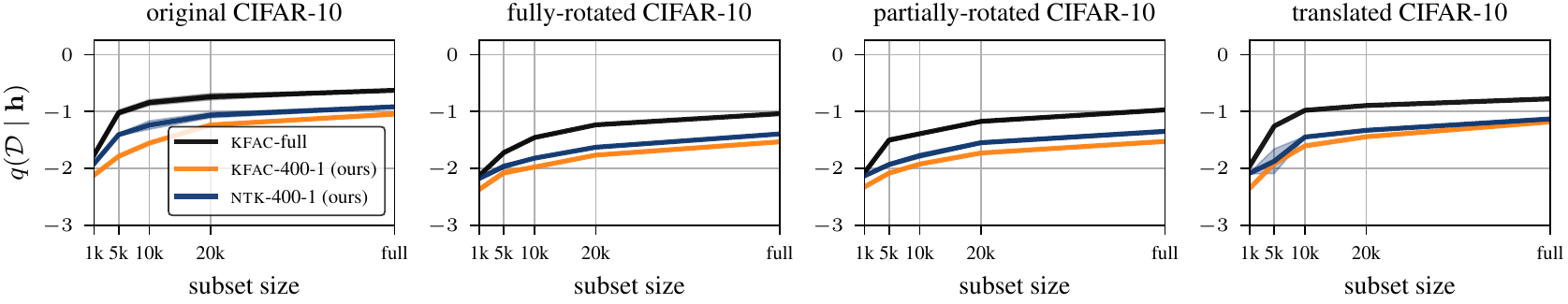}
    \vspace{-2em}
    \caption{Comparison of marginal likelihoods on different subsets of CIFAR-10 datasets.}
    \label{fig:subsets_marglik}
\end{figure*}
\begin{figure*}[h]
    \centering
    \includegraphics[width=\textwidth]{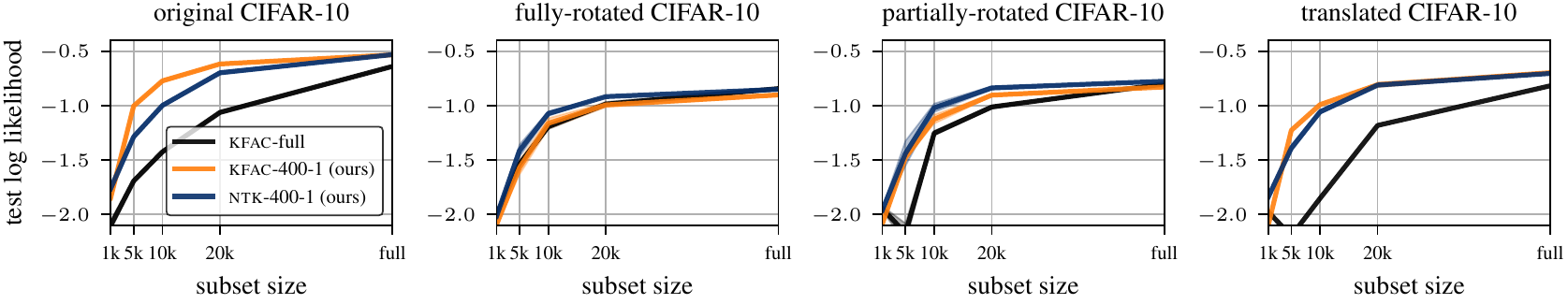}
    \vspace{-2em}
    \caption{Comparison of test log-likelihood on different subsets of CIFAR-10 datasets.}
    \label{fig:subsets_nll}
\end{figure*}
\begin{figure*}[h]
    \centering
    \includegraphics[width=\textwidth]{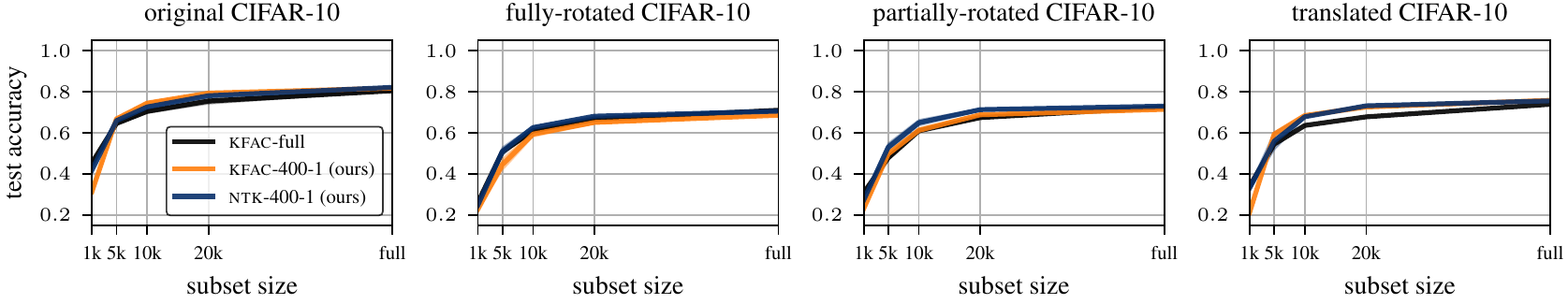}
    \vspace{-2em}
    \caption{Comparison of test accuracy on different subsets of CIFAR-10 datasets.}
    \label{fig:subsets_perf}
\end{figure*}

\end{document}